\documentclass{article}
\usepackage[final, main]{neurips_2026}
\makeatletter
\@ifundefined{algorithmic}{}{}
\@ifundefined{endalgorithmic}{}{}
\makeatother

\usepackage{multirow}
\usepackage{float} 
\usepackage{array}
\usepackage{amsmath, amssymb, amsthm}
\usepackage{algorithm}     
\usepackage{algpseudocode} 
\newtheorem{theorem}{Theorem}
\newtheorem{proposition}{Proposition}
\newtheorem{lemma}{Lemma}
\newtheorem{corollary}{Corollary}
\newtheorem{assumption}{Assumption}
\usepackage[table]{xcolor}
\definecolor{rpbred}{RGB}{255,230,230} 

\usepackage{mathtools}
\usepackage{bbm} 

\usepackage{graphicx}
\usepackage{subcaption}
\usepackage{placeins}
\usepackage{booktabs}
\usepackage{siunitx}
\newcommand{\res}[2]{$#1_{\pm #2}$}
\newcommand{\bres}[2]{$\mathbf{#1}_{\pm \mathbf{#2}}$}

\usepackage{enumitem}

\usepackage{fvextra}

\usepackage{microtype}

\usepackage{titlesec}
\titlespacing*{\section}{0pt}{*0.8}{*0.25}
\titlespacing*{\subsection}{0pt}{*0.45}{*0.2}
\titlespacing*{\subsubsection}{0pt}{*0.35}{*0.15}
\AtBeginDocument{%
  \setlength{\textfloatsep}{8pt plus 2pt minus 2pt}%
  \setlength{\floatsep}{6pt plus 2pt minus 2pt}%
  \setlength{\intextsep}{8pt plus 2pt minus 2pt}%
  \setlength{\dbltextfloatsep}{8pt plus 2pt minus 2pt}%
  \setlength{\dblfloatsep}{6pt plus 2pt minus 2pt}%
  \setlength{\abovecaptionskip}{3pt}%
  \setlength{\belowcaptionskip}{1pt}%
  \setlength{\abovedisplayskip}{3pt plus 1pt minus 1pt}%
  \setlength{\belowdisplayskip}{3pt plus 1pt minus 1pt}%
  \setlength{\abovedisplayshortskip}{1pt plus 1pt}%
  \setlength{\belowdisplayshortskip}{1pt plus 1pt}%
  \setlength{\jot}{2pt}%
}
\setlist{topsep=2pt,partopsep=0pt,itemsep=1pt,parsep=0pt}

\usepackage{hyperref}

\newcommand{\E}{\mathbb{E}}
\newcommand{\KL}{\mathrm{KL}}
\newcommand{\Var}{\mathrm{Var}}
\newcommand{\Hc}{\mathrm{H}} 
\newcommand{\R}{\mathbb{R}}

\DeclareMathOperator*{\argmin}{arg\,min}

\DefineVerbatimEnvironment{promptverb}{Verbatim}{
  breaklines=true,
  breakanywhere=true,
  fontsize=\small
}

\begin{document}

\title{Structural Rationale Distillation via Reasoning Space Compression}

\author{%
  \textbf{Jialin Yang}\textsuperscript{1,3}\thanks{Equal contribution.}\quad
  \textbf{Jiankun Wang}\textsuperscript{2}\footnotemark[1]\quad
  \textbf{Jiajun Wu}\textsuperscript{1,3}\\
  \textbf{Henry Leung}\textsuperscript{1}\quad
  \textbf{Jiayu Zhou}\textsuperscript{2}\quad
  \textbf{Steve Drew}\textsuperscript{1,3}\thanks{Corresponding author: \texttt{steve.drew@ucalgary.ca}}\\[3pt]
  \textsuperscript{1}University of Calgary\quad
  \textsuperscript{2}University of Michigan\quad
  \textsuperscript{3}TAIH
}

\maketitle
\begin{abstract}
When distilling reasoning from large language models (LLMs) into smaller ones, teacher rationales for similar problems often vary wildly in structure and strategy. Like a chef who makes the same dish differently each time, this inconsistency burdens the student with noisy supervision that is hard to internalize. We propose Distillation through Reasoning Path Compression (D-RPC), which constrains the teacher to follow a compact, dynamically maintained bank of reusable high-level reasoning paths. For each training question, D-RPC retrieves the most relevant path and conditions the teacher to follow it, producing rationales that are consistent across similar problems yet diverse enough to cover different problem types.
A PAC-Bayes analysis formalizes a trade-off between bank size and coverage, with smaller banks limiting path diversity at the risk of coverage gaps. Our ablations favor an intermediate bank size, consistent with this trade-off. Across five math and commonsense reasoning benchmarks with two student models, D-RPC achieves the highest average accuracy among the evaluated distillation methods, while using fewer tokens than template-heavy alternatives.
\end{abstract}

\section{Introduction}


Knowledge distillation trains a small language model (SLM) to reason by
imitating a large language model (LLM) teacher~\citep{distill_reasoning_slm,
SLM_reasoning}. However, the teacher's rationales for similar problems often
vary substantially in structure and strategy. An apprentice chef shown a different recipe for the same dish each day learns more slowly, even when every recipe is valid. Distillation faces the same dynamic. Asked the same type of math problem
repeatedly, the teacher may generate a ``compute rate, then combine quantities''
solution once and an ``estimate total, then subtract difference'' solution the
next~\citep{cotNoisyInconsistent, CoTNotFaithful}. This \emph{rationale
divergence} has three costs. First, the student sees many one-off patterns
instead of reusable strategies, making internalization
harder~\citep{chen2025unveilingkeyfactorsdistilling}. Second, covering all
valid solution paths requires many teacher generations, inflating
cost~\citep{STaR}. Third, highly variable supervision can amplify
hallucination during fine-tuning~\citep{distill_reasoning_slm}.

Prior work has tried to mitigate this variance through structured reasoning. Approaches like Buffer-of-Thought~\citep{yang2024bufferthoughtsthoughtaugmentedreasoning} and ReasonFlux~\citep{yang2025reasonflux} reuse hierarchical thought templates to improve LLM reasoning at inference time. However, these verbose templates impose heavy context loads that SLMs, with their limited capacity and context windows, struggle to exploit~\citep{SLM_reasoning, complexpromptBenchmark}. The core question remains open: \emph{can we compress the teacher's reasoning space into a compact, consistency-preserving structure before distilling it to a student?}

We propose Distillation through Reasoning Path Compression (D-RPC), which guides teacher rationale generation with a compact bank of canonical, high-level reasoning paths.\footnote{Code and data are available at \url{https://github.com/denoslab/D-RPC}.} Just as a chef who consults a cookbook produces more consistent dishes than one who improvises freely, a teacher conditioned on a shared set of reasoning paths produces more consistent supervision. D-RPC works in two stages.
First, it builds the cookbook: a seed set of teacher-solved problems is clustered by intent to identify canonical solution strategies, forming the \emph{reasoning path bank}.
Second, it teaches from the cookbook: for each new training question, the most relevant paths are retrieved from the bank, and the teacher generates a detailed rationale conditioned on the retrieved path.
If the teacher discovers a novel valid solution not in the bank, it is buffered and periodically merged back so that the cookbook grows over time.
The student SLM is then fine-tuned via LoRA on the resulting consistent supervision.

Our main contributions are as follows.
\begin{enumerate}[leftmargin=*, itemsep=2pt, topsep=2pt]
\item We propose D-RPC, which compresses the teacher's reasoning space into a compact, dynamically maintained reasoning path bank. By routing similar questions to shared canonical paths, D-RPC reduces rationale divergence and provides more consistent supervision for SLMs.
\item We formalize a trade-off between bank size and coverage via a PAC-Bayes generalization bound with an explicit $\log K_{\mathrm{bank}}$ term. Smaller banks tighten the path-entropy bound but can increase coverage slack, a trade-off reflected in the bank-size ablation.
\item Across five benchmarks, D-RPC achieves the highest average accuracy for both Llama~3.1 8B Instruct and Qwen~3 1.7B among the evaluated methods. It also uses fewer tokens than template-heavy alternatives.
\end{enumerate}

\section{Related Work}

\noindent\textbf{Distilling reasoning from teachers to students.}
UNICOTT iteratively constructs structured explanations to reduce supervision noise~\citep{zhuang2025unicott}. COTCD introduces difficulty-aware curricula for stable student learning~\citep{cot_Curriculum_distill}, and RLKD captures implicit multi-branch structures instead of mimicking a single trajectory~\citep{xu2025distillingimplicitmultibranchstructure}. ReasonFlux-PRM scores intermediate steps to filter distillation data and provide dense rewards~\citep{zou2025reasonflux}. MCC-KD generates multiple rationales per question and minimizes a bidirectional KL divergence across their answer distributions to enforce \emph{intra-question} consistency~\citep{chen2023mcckdmulticotconsistentknowledge}, and MIND addresses misalignment between a teacher's ``optimal'' rationale and the student's evolving capacity by synthesizing capability-adaptive multi-perspective supervision through a teaching-assistant network~\citep{cui2026mindpassivemimicryactive}. D-RPC focuses on structural reuse across related questions. Its compact path bank guides the teacher before rationale generation. This approach can complement methods that adapt supervision to the student's current abilities.

\noindent\textbf{Reusing reasoning processes at inference time.}
A parallel line of work treats reasoning as a reusable process rather than a one-off generation. Buffer of Thought and ReasonFlux reuse thought templates~\citep{yang2024bufferthoughtsthoughtaugmentedreasoning,yang2025reasonflux}. RoT retrieves thought steps from similar problems~\citep{ahmed2025retrievalofthoughtefficientreasoningreusing}, and Thought Propagation reuses solutions through analogical propagation~\citep{yu2024thoughtpropagationanalogicalapproach}. In the agent setting, ReasoningBank~\citep{ouyang2026reasoningbankscalingagentselfevolving} accumulates a memory of natural-language strategies extracted from past successful and failed trajectories and merges them into the agent's prompt, with no weight updates. All of these methods improve LLM or agent performance at inference time by strengthening prompts. However, they do not address the variance of the supervision signal used to train smaller students. D-RPC instead operates at \emph{distillation time}, compressing the teacher's supervision distribution so that consistent reasoning structure is internalized into the student's weights and no retrieval is required at deployment.

\noindent\textbf{Reasoning-path supervision for SFT.}
Skip-Thinking refines CoTs through block training, allowing non-inference blocks to be skipped~\citep{chen2025skipthinkingchunkwisechainofthoughtdistillation}. EDIT identifies key decision-making steps by comparing paired CoTs and upweights them~\citep{dai-etal-2025-capture}. QR-Distill improves multi-path refinement through quality filtering and peer model teaching~\citep{lei2025learningdiversereasoningpaths}. Reasoning Scaffolding abstracts teacher rationales into sequences of discrete semantic signals and trains the student to predict them alongside each step~\citep{wen2025reasoningscaffoldingdistillingflow}, enforcing structure within each trace rather than across related questions. D-RPC uses a shared bank to encourage reusable structures during teacher generation. This design can be combined with rationale filtering, step reweighting, or scaffolding during student training.

\section{Methodology}
\label{sec:method}

Returning to our cooking analogy, D-RPC builds a cookbook, teaches from it, and trains the student.
In the first stage, it assembles a compact bank of canonical, high-level reasoning paths by having the teacher solve a small seed set and clustering the resulting solution strategies.
In the second stage, it retrieves the best-matching path for each new training question, conditions the teacher on that path, and collects the resulting supervision tuples. In the third stage, the student SLM is fine-tuned on the collected supervision via LoRA.
Figure~\ref{fig:methodology} illustrates the full pipeline; Algorithm~\ref{alg:drpc_compact} gives the pseudocode. In short, D-RPC takes a teacher LLM and a question set and produces two outputs: a reasoning path bank and a LoRA-trained student.
Everything in between is bank construction (Stage~1), bank-guided supervision (Stage~2), and student training (Stage~3).

\begin{figure*}[t]
  \centering
  \includegraphics[width=\textwidth]{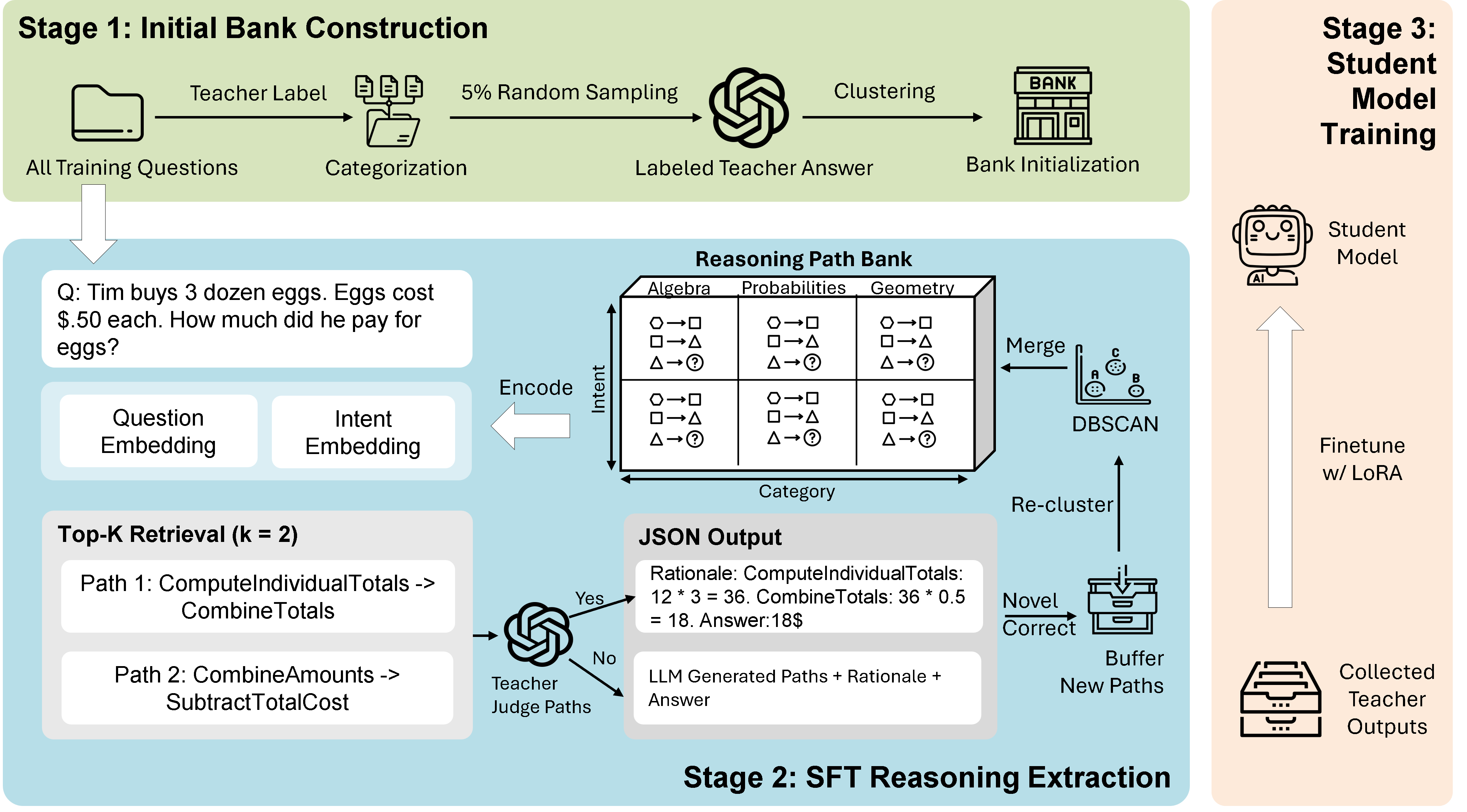}
  \caption{D-RPC pipeline.
  {Stage~1, Reasoning bank initialization.} A seed subset of questions is solved by the teacher; the resulting reasoning paths are clustered by intent to form the initial reasoning path bank.
  {Stage~2, Bank-guided supervision.} For each training question, retrieve the top-$K_{\mathrm{ret}}$ candidate paths, condition the teacher to follow one, and collect the supervision tuple. Novel correct paths are buffered and periodically merged back into the bank via re-clustering.
  {Stage~3, Student model training.} The student SLM is fine-tuned with LoRA on the collected supervision tuples.}
  \label{fig:methodology}
  \vspace{-0.05in}
\end{figure*}

\subsection{Notation and setup}
\label{subsec:setup}
Let $\{X_i\}_{i=1}^{N}$ be a dataset of training questions.
The teacher LLM labels each question with a \emph{category} $C_i$ such as ``Arithmetic'' and a finer-grained \emph{intent} $T_i$ such as ``compute unit rate''.
For distillation, the teacher produces a supervision tuple $Y_i = (\Pi_i, R_i, A_i)$ where $\Pi_i$ is a \emph{reasoning path}, meaning a short ordered sequence of abstract steps such as ``compute the unit rate, then combine quantities''; $R_i$ is a \emph{detailed rationale} that instantiates the path with concrete computations; and $A_i$ is the final answer.
We write $\mathcal{P}$ for the universe of high-level reasoning paths (ordered sequences of abstract steps) and $\mathcal{C}$ for the set of all categories.
The student is trained on $(X_i, Y_i)$ only; category and intent labels are used only to build and query the bank.

\subsection{Reasoning bank initialization}
\label{subsec:stage1}

\noindent\textbf{Categorize all questions.}
The teacher assigns a category--intent pair to every question: $(C_i, T_i) = f_{\mathrm{cat}}(X_i)$.
We then draw a category-balanced seed set $\mathcal{S}_{\mathrm{seed}}$ (5\% of training data in our experiments), sampling proportional to $\hat p(C)$ so that every problem type is represented.

\noindent\textbf{Elicit reasoning paths from the seed set.}
For each seed question $X_i \in \mathcal{S}_{\mathrm{seed}}$, the teacher generates a high-level reasoning path and answer: $(\Pi_i, A_i) = f_{\mathrm{path}}(X_i, C_i, T_i)$.

\noindent\textbf{Cluster intents to form the bank.}
Within each category $c$, we embed the teacher-generated intent strings via $e(t) = \mathrm{Embed}(t)$ and cluster them with DBSCAN~\citep{ester1996density}.
Each cluster yields a \emph{canonical intent} $\tilde t$---a representative label for a group of semantically similar intents; for example, ``unit-rate computation,'' ``compute unit rate,'' and ``rate-times-quantity'' may collapse into the single canonical intent ``unit-rate computation.''
We denote the set of canonical intents for category $c$ by $\tilde{\mathcal{T}}(c)$.
The resulting \emph{reasoning path bank} $\mathcal{B}$ maps each (category, canonical intent) pair to its set of observed reasoning paths:
$\mathcal{B}(c, \tilde t) \subseteq \mathcal{P}$.
The total number of distinct paths stored across all bank entries is
$K_{\mathrm{bank}} := \bigl|\bigcup_{c \in \mathcal{C}}\bigcup_{\tilde t \in \tilde{\mathcal{T}}(c)} \mathcal{B}(c, \tilde t)\bigr|$.

\subsection{Bank-guided supervision}
\label{subsec:stage2}

\noindent\textbf{Routing and retrieval.}
For a new training question $X_j$, the teacher predicts $(C_j, T_j) = f_{\mathrm{cat}}(X_j)$.
Because the predicted intent string $T_j$ can be noisy~\citep{cotNoisyInconsistent}, we match $X_j$ to a canonical intent via the question embedding rather than the surface form of $T_j$:
$\tilde T_j = \arg\max_{\tilde t \in \tilde{\mathcal{T}}(C_j)} \cos\bigl(e(X_j), e(\tilde t)\bigr)$.
We then retrieve the top-$K_{\mathrm{ret}}$ candidate paths from the matched bank entry:
$\mathcal{P}^{(K_{\mathrm{ret}})}(X_j) = \operatorname{TopK}_{\Pi \in \mathcal{B}(C_j, \tilde T_j)} \cos\bigl(e(X_j), e(\Pi)\bigr)$.

\noindent\textbf{Path-guided rationale generation.}
The teacher receives the question and the retrieved candidate paths, picks the best-fitting one (or conservatively refines one if none fits directly), and generates a detailed rationale following it:
$(\Pi_j, R_j, A_j) = f_{\mathrm{teach}}\bigl(X_j, \mathcal{P}^{(K_{\mathrm{ret}})}(X_j)\bigr)$,
where $\Pi_j$ is the path actually used. Details of the teacher prompt are in Appendix~\ref{app:prompt-templates}.
Every resulting tuple $(X_j, (\Pi_j, R_j, A_j))$ is added to the supervision dataset $\mathcal{D}_{\mathrm{SFT}}$ regardless of answer correctness.

\noindent\textbf{Dynamic bank refinement.}
If the teacher produces a path $\Pi_j \notin \bigcup_{c,\tilde t}\mathcal{B}(c,\tilde t)$ and the answer is correct, we buffer the record $(C_j, T_j, \Pi_j)$ into a set $\mathcal{M}$.
Once $|\mathcal{M}| \ge \tau_{\mathrm{buf}}$, we re-cluster: the raw intents $T_j$ are merged into the canonical intent clusters $\tilde{\mathcal{T}}(c)$, and the new paths are added to the corresponding bank entries $\mathcal{B}(c, \tilde t)$.
This allows the bank to grow beyond its initial coverage.

\subsection{Student fine-tuning}
\label{subsec:lora}
\vspace{-0.1in}

With $\mathcal{D}_{\mathrm{SFT}}$ collected, we fine-tune the student SLM using LoRA~\citep{hu2021loralowrankadaptationlarge}: the backbone is frozen and only low-rank adapters $\theta$ are trained.
The objective is the standard supervised NLL:
\[
\mathcal{L}_{\mathrm{SFT}}
= \E_{(X,Y) \sim \mathcal{D}_{\mathrm{SFT}}}\!\left[-\log p_{\theta}(Y \mid X)\right],
\qquad Y = (\Pi, R, A).
\]
Because related questions share reasoning paths, the supervision signal has lower variance than unconstrained teacher generation, making it easier for the student to internalize reusable solution strategies. The full D-RPC algorithm is given in Algorithm~\ref{alg:drpc_compact}.

\begin{algorithm}[t]
\small
\caption{ Distillation through Reasoning Path Compression (D-RPC)}
\label{alg:drpc_compact}
\begin{algorithmic}[1]
\Require $\!\!$Questions $\!\{\!X_i\!\}\!_{i=1}^N$; teacher function $\!f_{\mathrm{cat}}, \!f_{\mathrm{path}}, \!f_{\mathrm{teach}}$;
encoder $\!e(\cdot)$; retrieval size $\!K_{\!\mathrm{ret}}$; buffer size $
\!\tau_{\mathrm{buf}}$.
\Ensure Bank $\mathcal{B}$; supervision $\mathcal{D}_{\mathrm{SFT}}$; LoRA params $\theta$.

\For{$i=1$ \textbf{to} $N$} \State $(C_i,T_i)\gets f_{\mathrm{cat}}(X_i)$ \EndFor
\State $\mathcal{S}_{\mathrm{seed}}\gets \Call{CategoryBalancedSample}{\{X_i\},\{C_i\}}$
\For{$X_i \in \mathcal{S}_{\mathrm{seed}}$} \State $(\Pi_i,A_i)\gets f_{\mathrm{path}}(X_i,C_i,T_i)$ \EndFor
\State $(\mathcal{B},\{\tilde{\mathcal{T}}(c)\})\gets \Call{InitBank}{\mathcal{S}_{\mathrm{seed}},\{(C_i,T_i,\Pi_i)\}}$
\State $\mathcal{D}_{\mathrm{SFT}}\gets\emptyset$;\ $\mathcal{M}\gets\emptyset$
\For{$j=1$ \textbf{to} $N$}
  \State $(C_j,T_j)\gets f_{\mathrm{cat}}(X_j)$
  \State $\tilde T_j\gets \arg\max_{\tilde t\in\tilde{\mathcal{T}}(C_j)} \cos(e(X_j),e(\tilde t))$
  \State $\mathcal{P}^{(K_{\mathrm{ret}})}\gets \operatorname{TopK}_{\Pi\in\mathcal{B}(C_j,\tilde T_j)} \cos(e(X_j),e(\Pi))$
  \State $(\Pi_j,R_j,A_j)\gets f_{\mathrm{teach}}(X_j,\mathcal{P}^{(K_{\mathrm{ret}})})$
  \State $\mathcal{D}_{\mathrm{SFT}}\gets \mathcal{D}_{\mathrm{SFT}}\cup\{(X_j,(\Pi_j,R_j,A_j))\}$
  \If{\textsc{IsCorrect}$(A_j)$ \textbf{and} $\Pi_j\notin\bigcup_{c,\tilde t}\mathcal{B}(c,\tilde t)$}
    \State $\mathcal{M}\gets \mathcal{M}\cup\{(C_j,T_j,\Pi_j)\}$
  \EndIf
  \If{$|\mathcal{M}|\ge \tau_{\mathrm{buf}}$}
    \State $\mathcal{B}\gets \Call{Recluster}{\mathcal{B},\mathcal{M}}$;\ $\mathcal{M}\gets\emptyset$
  \EndIf
\EndFor
\State $\theta\gets \arg\min_{\theta}\ \mathbb{E}_{(X,Y)\sim\mathcal{D}_{\mathrm{SFT}}}\big[-\log p_{\theta}(Y\mid X)\big]$
\State \Return $\mathcal{B},\theta$
\end{algorithmic}
\end{algorithm}

\section{Theoretical Analysis}
\label{sec:theory}

Why does a moderate-sized bank work best?
In learning terms, an overly small bank induces high approximation bias because many examples are assigned to mismatched path supports, whereas an overly large bank increases supervision entropy and reduces parameter sharing across examples.
The best regime balances coverage against compression.
This section formalizes that intuition.
We first state the main result, then unpack the machinery behind it.

\textbf{Scope.} The analysis below applies specifically to the LoRA-based SFT instantiation described in Section~\ref{sec:method}. Proposition~\ref{prop:entropy-baseline}, which controls entropy via banking, is general, but the KL-control argument in Propositions~\ref{prop:norm-control}--\ref{prop:kl-bound} relies on the LoRA parameterization. Convergence and information-theoretic guarantees for low-rank distillation more broadly have been studied in~\citep{soarez2026demystifyinglowrankknowledgedistillation}; our result is complementary in that the supervision-entropy floor is controlled explicitly through the bank-size term $\log K_{\mathrm{bank}}$, which their analysis does not consider. Extending to other fine-tuning regimes would require a different complexity argument.

\subsection{Main result: the trade-off between bank size and coverage}
\label{subsec:main-thm}

\begin{theorem}[Banked LoRA distillation bound]
\label{thm:main}
Under Assumptions~\ref{assump:bounded-nll}--\ref{assump:reg-erm} (stated below),
let \(\hat\theta_S\) be the regularized LoRA minimizer and define a Gaussian posterior \(\mathsf Q_S = \mathcal{N}(\hat\theta_S, \sigma_{\mathrm{post}}^2 I_m)\) with prior \(\mathsf P = \mathcal{N}(0, \sigma_0^2 I_m)\).
Then for any \(\delta \in (0,1)\), with probability at least \(1-\delta\) over the training set~\(S\),
\begin{equation}
L(\mathsf Q_S)
\;\le\;
\underbrace{\widehat L_S(\mathsf Q_S)}_{\text{training loss}}
+
\underbrace{\sqrt{\frac{2\tau \cdot M \cdot N}{n}}}_{\text{generalization gap}}
+
\underbrace{\frac{c\tau \cdot N}{n}}_{\text{lower-order}},
\label{eq:main}
\end{equation}
where $c > 0$ is a universal constant, $\tau$ bounds the per-example loss, and:
\begin{align}
M &:= \log K_{\mathrm{bank}} + \E[\Hc(R\mid X,\Pi)] + \E[\Hc(A\mid X,\Pi,R)] + \varepsilon,
\label{eq:def-M} \\[3pt]
N &:= \underbrace{\tfrac{1}{\sigma_0^2\lambda}\,\widehat L_S(0)}_{\mathclap{\text{frozen-model fit}}}
+ \tfrac{m}{2}(\rho - 1 - \ln\rho)
+ \ln\tfrac{1}{\delta}.
\label{eq:def-N}
\end{align}
\end{theorem}

The generalization gap scales with $\sqrt{M \cdot N}$.
The bank size $K_{\mathrm{bank}}$ enters through $M$ via $\log K_{\mathrm{bank}}$, creating two opposing forces.
Making the bank smaller reduces $\log K_{\mathrm{bank}}$, so supervision is more consistent, but increases the coverage slack $\varepsilon$ because some problems have no good path, raising $M$ from the other side.
Conversely, making the bank larger reduces $\varepsilon$ by improving coverage, but increases $\log K_{\mathrm{bank}}$ by weakening consistency, also raising $M$.
An intermediate bank size can balance reuse and coverage by lowering $M(K_{\mathrm{bank}})$.
The accuracy pattern in Table~\ref{tab:bank_size_paths_std} is consistent with this trade-off.
Banking can also reduce $N$ by lowering the frozen-model loss $\widehat L_S(0)$, since banked rationales are more predictable to the pretrained student.

\subsection{Assumptions and supporting results}
\label{subsec:building-blocks}

We now state the assumptions and intermediate results that compose into Theorem~\ref{thm:main}. Full proofs are deferred to the appendix.

\noindent\textbf{Setup.}
The student is an autoregressive model $p_\theta(y \mid x)$ with frozen backbone; only LoRA weights $\theta \in \R^m$ are trained.
The loss is the NLL: $\ell(\theta; z) = -\log p_\theta(y \mid x)$.
Population and empirical risks are $L(\theta) = \E[\ell(\theta; Z)]$ and $\widehat L_S(\theta) = \frac{1}{n}\sum_i \ell(\theta; Z_i)$.
For a distribution $\mathsf{Q}$ over $\theta$, the Gibbs risk is $L(\mathsf{Q}) = \E_{\theta \sim \mathsf{Q}}[L(\theta)]$.

\begin{assumption}[Bounded NLL]
\label{assump:bounded-nll}
$0 \le \ell(\theta; Z) \le \tau$ a.s.\ for all $\theta$ in the posterior support (justified by length normalization or probability lower bounds).
\end{assumption}

\noindent\textbf{Part I: Banking controls supervision entropy (general).}

\begin{assumption}[Path-bank support]
\label{assump:bank-support}
Under the banked regime, $\Pi \in \mathcal{B}_{\mathrm{path}}$ a.s.\ given $X$, where $\mathcal{B}_{\mathrm{path}} := \bigcup_{c,\tilde t} \mathcal{B}(c,\tilde t)$ is the global path set and $|\mathcal{B}_{\mathrm{path}}| = K_{\mathrm{bank}}$.
\end{assumption}

\begin{proposition}[Banking reduces path uncertainty]
\label{prop:entropy-baseline}
Under Assumption~\ref{assump:bank-support}, the conditional entropy of the supervision satisfies
\begin{equation}
\E[\Hc(Y \mid X)] \;\le\; \log K_{\mathrm{bank}} + \E[\Hc(R \mid X, \Pi)] + \E[\Hc(A \mid X, \Pi, R)].
\label{eq:entropy-baseline}
\end{equation}
\end{proposition}
\noindent
This follows from the chain rule: $\Hc(Y \mid X) = \Hc(\Pi \mid X) + \Hc(R \mid X, \Pi) + \Hc(A \mid X, \Pi, R)$, with $\Hc(\Pi \mid X) \le \log K_{\mathrm{bank}}$ by finite support (proof in Appendix~\ref{app:entropy-proof}).

\begin{assumption}[Near-floor empirical fit]
\label{assump:near-floor}
There exists a posterior $\mathsf{Q}$ such that
\begin{equation}
\widehat L_S(\mathsf{Q}) \le \E[\Hc(Y \mid X)] + \varepsilon, \qquad \varepsilon \ge 0.
\label{eq:near-floor}
\end{equation}
This is an optimization-dependent condition: it says the trained model nearly matches the information-theoretic floor.
We verify empirically that this holds for our AQUA experiments in Section~\ref{subsec:bound-validation} (Table~\ref{tab:bound-components}).
\end{assumption}

Combining Proposition~\ref{prop:entropy-baseline} with Assumption~\ref{assump:near-floor} and a standard Bernstein-type PAC-Bayes bound, proved in Appendix~\ref{app:pacbayes-proof}, yields a generalization bound where $K_{\mathrm{bank}}$ explicitly controls the deviation term through $M$ as defined in Eq.~\ref{eq:def-M}.

\noindent\textbf{Part II: Banking can reduce the complexity term (LoRA-specific).}

\begin{assumption}[Regularized LoRA ERM]
\label{assump:reg-erm}
Training returns a minimizer of the regularized objective:
$\hat\theta_S \in \argmin_{\theta} [\widehat L_S(\theta) + \tfrac{\lambda}{2}\|\theta\|^2]$, $\lambda > 0$.
\end{assumption}

\begin{proposition}[Norm control]
\label{prop:norm-control}
Under Assumption~\ref{assump:reg-erm}:
\begin{equation}
\|\hat\theta_S\|^2 \le \frac{2}{\lambda}\widehat L_S(0).
\label{eq:norm-control}
\end{equation}
\end{proposition}
\noindent\textit{Proof.} Optimality of $\hat\theta_S$ gives $\widehat L_S(\hat\theta_S) + \frac{\lambda}{2}\|\hat\theta_S\|^2 \le \widehat L_S(0)$; since $\widehat L_S(\hat\theta_S)\ge 0$, the result follows. $\square$

\begin{proposition}[KL bound]
\label{prop:kl-bound}
For Gaussian prior $\mathsf{P}=\mathcal{N}(0,\sigma_0^2 I)$ and posterior $\mathsf{Q}_S=\mathcal{N}(\hat\theta_S,\sigma_{\mathrm{post}}^2 I)$ with $\rho=\sigma_{\mathrm{post}}^2/\sigma_0^2$:
\begin{equation}
\KL(\mathsf{Q}_S\|\mathsf{P}) \le \frac{1}{\sigma_0^2\lambda}\widehat L_S(0) + \frac{m}{2}(\rho-1-\ln\rho).
\label{eq:kl-bound}
\end{equation}
Proof in Appendix~\ref{app:kl-proof}.
\end{proposition}

The KL term depends on the frozen-model loss $\widehat L_S(0)$.
Banking can reduce $\widehat L_S(0)$ and hence the complexity penalty $N$ in Theorem~\ref{thm:main}.

\begin{proposition}[Sufficient conditions for frozen-model alignment]
\label{prop:frozen-align}
Under any of the following conditions, banking reduces the frozen-model population cross-entropy ($\mathcal{L}_0^{(\star)} := \E_X \E_{Y \sim p^\star(\cdot \mid X)}[-\log p_0(Y \mid X)]$ for regime $\star \in \{\mathrm{bank}, \mathrm{free}\}$):
\begin{itemize}[nosep,leftmargin=2.5em,labelsep=0.5em]
\item[\textbf{D1.}] \emph{Bank as projection:} the banked target conditional minimizes the frozen-model cross-entropy over all distributions supported on $K_{\mathrm{bank}}$ paths.
\item[\textbf{D2.}] \emph{Canonicalization improves frozen likelihood pointwise:} there exists a map $c$ on targets such that $-\log p_0(c(y)\mid x) \le -\log p_0(y\mid x)$ for all $(x,y)$, and the banked regime restricts support to such canonical forms.
\item[\textbf{D3.}] \emph{Entropy drop dominates mismatch increase:} the reduction in $\E[\Hc(Y\mid X)]$ from banking exceeds any increase in $\E[\KL(p^\star(\cdot\mid X)\|p_0(\cdot\mid X))]$.
\end{itemize}
Formal statements and proofs are in Appendix~\ref{app:frozen-align}.
The conclusion of this proposition is directly testable; Section~\ref{subsec:bound-validation} confirms that $\widehat L_S(0)$ is indeed lower under banking than under free generation.
\end{proposition}

\noindent\textbf{Proof of Theorem~\ref{thm:main}.}
Substitute the KL bound (Proposition~\ref{prop:kl-bound}) into the $K_{\mathrm{bank}}$-explicit PAC-Bayes bound and collect terms into $M$ and $N$.
Full derivation in Appendix~\ref{app:main-proof}. $\square$

\section{Experiments}
\label{sec:experiments}

\subsection{Datasets}
We evaluate on five benchmarks (see details in Appendix~\ref{app:detailed-datasets}): GSM8K~\citep{gsm8k} for grade-school arithmetic, AQUA~\citep{AQUA} for multiple-choice math, StrategyQA~\citep{strategyQA} for multi-hop commonsense, AI2ARC~\citep{ai2arc} for science reasoning, and MATH~\citep{math_benchmark} for competition-level mathematics. These range from relatively simple, such as GSM8K, to challenging, such as MATH, allowing us to test whether reasoning-path compression scales across difficulty levels.

\subsection{Implementation Details}
\label{subsec:impl-details}

\noindent\textbf{Clustering and embedding.}
Intent strings are embedded with \texttt{all-MiniLM-L6-v2}~\citep{reimers2019sentencebertsentenceembeddingsusing} and clustered via DBSCAN (\texttt{eps}$\,{=}\,0.25$, \texttt{min\_samples}$\,{=}\,2$). No extensive sweep over clustering or embedding choices was performed; we discuss this limitation in Section~\ref{sec:conclusion}.

\noindent\textbf{Evaluation protocol.}
In the main comparisons, all methods use the final checkpoint (epoch~2.0, no validation-based selection), identical LoRA configuration ($r{=}64$, $\alpha{=}128$), and a maximum output length of 2048 tokens. Unless otherwise specified, we report means and standard deviations over 10 repeated executions with the random seed fixed to 42. The standard deviations capture variability from nondeterministic GPU execution. More details are in Appendix~\ref{training_details}.

\noindent\textbf{Metrics.}
We report \emph{accuracy} (Acc) (exact match against gold; unparseable outputs count as incorrect) and \emph{format validity} (FV) (accuracy among outputs that parse into valid JSON).

\subsection{Baselines}
We compare D-RPC against five distillation strategies, all sharing the same teacher (GPT-5.1) and SFT configuration, evaluated on two student models, i.e., Llama~3.1 8B Instruct and Qwen~3 1.7B.
As an un-distilled reference, \textbf{Direct Prompting} evaluates the raw student without fine-tuning.
Among unstructured approaches, \textbf{Zero-Shot CoT}~\citep{zero_shot_cot} distills chain-of-thought rationales and \textbf{Freeform} lets the teacher freely determine step count and granularity without predefined paths.
Among structured approaches, \textbf{DCoT}~\citep{dcot} directly distills chain-of-thought outputs, \textbf{SGFT}~\citep{sgft} fine-tunes on structured rationales without retrieval, and \textbf{SuperCorrect}~\citep{supercorrect} uses hierarchical XML thought templates (math-only; not applied to StrategyQA or AI2ARC).
Training details and prompts are in Appendices~\ref{training_details} and~\ref{app:prompt-templates}.

\subsection{Main Results}

Table~\ref{tab:all_metrics_with_delta} reports absolute post-distillation accuracy across all five benchmarks for both student models.
We report absolute accuracy rather than gains over each method's own Direct Prompting baseline, because Direct Prompting baselines differ across methods due to different prompting formats and relative gains can be misleading.
The impact of LoRA rank and $\alpha$ is discussed in Appendix~\ref{app:lora_impact}.

\begin{table}[t]
\caption{Post-distillation accuracy (\%) across five benchmarks and two student models. Values are mean $\pm$ std over 10 repeated executions with seed 42. \textbf{Bold} marks the best result per column. \colorbox{rpbred}{Shading} marks D-RPC. Dashes mark inapplicable settings.}
\label{tab:all_metrics_with_delta}
\centering
\small
\setlength{\tabcolsep}{4.5pt}

\begin{tabular}{lcccccc}
\toprule
& \multicolumn{6}{c}{\textbf{Student: Llama 3.1 8B Instruct} ($r=64$)} \\
\cmidrule(lr){2-7}
Strategy & GSM8K & AQUA & StrategyQA & AI2ARC & MATH & Avg. \\
\midrule
CoT
& 83.96 ($\pm$0.33)
& 64.02 ($\pm$0.77)
& 73.32 ($\pm$1.31)
& 87.92 ($\pm$0.29)
& 41.57 ($\pm$0.35)
& 70.16 \\
Freeform
& 81.63 ($\pm$0.72)
& 60.39 ($\pm$1.61)
& 72.27 ($\pm$1.27)
& 92.41 ($\pm$0.34)
& 41.68 ($\pm$0.18)
& 69.68 \\
SuperCorrect
& 82.42 ($\pm$0.46)
& 59.45 ($\pm$2.59)
& --
& --
& 36.62 ($\pm$0.46)
& -- \\
DCoT
& 81.86 ($\pm$0.64)
& 61.81 ($\pm$2.43)
& 72.53 ($\pm$1.81)
& 80.78 ($\pm$12.38)
& 45.23 ($\pm$0.53)
& 68.44 \\
SGFT
& 77.94 ($\pm$0.69)
& 60.16 ($\pm$2.67)
& 70.31 ($\pm$1.29)
& 86.71 ($\pm$0.76)
& 36.15 ($\pm$0.43)
& 66.25 \\
\cellcolor{rpbred}D-RPC (Ours)
& \cellcolor{rpbred}\textbf{85.41} ($\pm$0.49)
& \cellcolor{rpbred}\textbf{67.52} ($\pm$1.59)
& \cellcolor{rpbred}\textbf{74.15} ($\pm$1.49)
& \cellcolor{rpbred}\textbf{92.92} ($\pm$0.34)
& \cellcolor{rpbred}\textbf{48.76} ($\pm$0.45)
& \cellcolor{rpbred}\textbf{73.75} \\
\bottomrule
\end{tabular}

\vspace{4pt}

\begin{tabular}{lcccccc}
\toprule
& \multicolumn{6}{c}{\textbf{Student: Qwen 3 1.7B} ($r=64$)} \\
\cmidrule(lr){2-7}
Strategy & GSM8K & AQUA & StrategyQA & AI2ARC & MATH & Avg. \\
\midrule
CoT
& 77.92 ($\pm$0.29)
& 59.92 ($\pm$0.89)
& 61.92 ($\pm$0.92)
& 88.78 ($\pm$0.40)
& 49.21 ($\pm$0.40)
& 67.55 \\
Freeform
& 73.38 ($\pm$0.68)
& 62.99 ($\pm$0.83)
& 61.75 ($\pm$2.29)
& 88.70 ($\pm$0.30)
& 48.09 ($\pm$0.26)
& 66.98 \\
SuperCorrect
& 76.44 ($\pm$0.58)
& 63.78 ($\pm$2.59)
& --
& --
& 42.82 ($\pm$0.36)
& -- \\
DCoT
& 73.99 ($\pm$0.69)
& 59.65 ($\pm$1.37)
& \textbf{66.16} ($\pm$1.39)
& 85.74 ($\pm$2.90)
& 54.91 ($\pm$0.34)
& 68.09 \\
SGFT
& 65.44 ($\pm$0.59)
& 49.06 ($\pm$10.60)
& 60.26 ($\pm$2.78)
& 75.99 ($\pm$0.18)
& 29.44 ($\pm$0.29)
& 56.04 \\
\cellcolor{rpbred}D-RPC (Ours)
& \cellcolor{rpbred}\textbf{78.29} ($\pm$0.95)
& \cellcolor{rpbred}\textbf{74.76} ($\pm$1.84)
& \cellcolor{rpbred}64.59 ($\pm$2.46)
& \cellcolor{rpbred}\textbf{88.82} ($\pm$0.54)
& \cellcolor{rpbred}\textbf{59.72} ($\pm$0.39)
& \cellcolor{rpbred}\textbf{73.24} \\
\bottomrule
\end{tabular}

\vspace{-0.1in}
\end{table}

\noindent\textbf{Llama 3.1 8B Instruct.}
D-RPC achieves the highest accuracy on all five benchmarks: 85.41\% on GSM8K, 67.52\% on AQUA, 74.15\% on StrategyQA, 92.92\% on AI2ARC, and 48.76\% on MATH. D-RPC outperforms the next-best method by +3.53 points on MATH and +3.50 on AQUA. D-RPC also surpasses SuperCorrect~\citep{supercorrect} at substantially lower token cost, as detailed in Appendix~\ref{app:token_cost}.

\noindent\textbf{Qwen 3 1.7B.}
The same pattern holds with a smaller, different-family student: D-RPC leads on GSM8K, AQUA, AI2ARC, and MATH, and is competitive on StrategyQA, where DCoT is slightly higher. The gains on MATH are especially large at $+10.51$ over CoT, demonstrating that path compression scales to challenging benchmarks and smaller models.

\subsection{Analysis}

\noindent\textbf{Why D-RPC helps.}
The gains on MATH and AQUA are consistent with our hypothesis that structural compression improves supervision across related questions.
The bank encourages reuse of high-level reasoning paths, while semantic retrieval selects suitable paths for each question.
This matching step provides an additional refinement whose benefit varies across models and tasks.
D-RPC also outperforms SuperCorrect with fewer tokens in the reported comparisons (Appendix~\ref{app:token_cost}).

\noindent\textbf{Path controls and answer filtering.}
Table~\ref{tab:path_controls} compares D-RPC with controls that preserve its prompt, output schema, and generation budget, using the conditions defined in Appendix~\ref{app:aqua_multiseed}.
On AQUA, all three path controls lower mean accuracy for both students.
On MATH, Empty and Random are slightly higher for both students, while Random-Global is lower for Llama and slightly higher for Qwen.
These small MATH gains over D-RPC are below the reported standard deviations across fixed-seed executions.
Filtered-SFT increases mean accuracy on MATH and decreases it on AQUA, with teacher answer-error rates reported in Appendix~\ref{app:teacher_errors}.

\begin{table}[htbp]
\centering
\caption{Path controls and filtered-SFT accuracy (\%). Means and sample standard deviations use ten executions with seed 42, evaluated on all 254 AQUA and 5000 MATH test questions. Llama and Qwen denote the students in Table~\ref{tab:all_metrics_with_delta}.}
\label{tab:path_controls}
\small
\begin{tabular}{lcccc}
\toprule
& \multicolumn{2}{c}{AQUA} & \multicolumn{2}{c}{MATH} \\
\cmidrule(lr){2-3} \cmidrule(lr){4-5}
Condition & Llama & Qwen & Llama & Qwen \\
\midrule
\rowcolor{rpbred}
D-RPC & \res{67.52}{1.59} & \res{74.76}{1.84} & \res{48.76}{0.45} & \res{59.72}{0.39} \\
Empty & \res{66.54}{1.16} & \res{72.24}{1.02} & \res{48.89}{0.42} & \res{60.07}{0.39} \\
Random & \res{65.75}{2.32} & \res{73.90}{1.80} & \res{48.89}{0.36} & \res{60.05}{0.49} \\
Random-Global & \res{65.43}{2.44} & \res{73.94}{2.49} & \res{48.34}{0.34} & \res{59.83}{0.41} \\
Filtered-SFT & \res{65.75}{2.12} & \res{74.25}{1.30} & \res{48.84}{0.45} & \res{60.42}{0.44} \\
\bottomrule
\end{tabular}
\end{table}

\noindent\textbf{Structural reuse in teacher supervision.}
Table~\ref{tab:path_structure} shows increasing reuse and decreasing path entropy from Empty through the random controls to D-RPC on both datasets.
The student results motivate extending the current embedding similarity retrieval with rationale quality and student learning needs, as discussed in Section~\ref{sec:limitations}.

\begin{table}[htbp]
\centering
\caption{Structure of teacher rationales. Each path is the ordered sequence of rationale step names, with empty paths excluded. $N$ counts nonempty rationales and $K_{\mathrm{used}}$ counts distinct paths. Reuse is $100(1-K_{\mathrm{used}}/N)$, and $\Hc_{\mathrm{path}}$ is the Shannon entropy of path frequencies within each dataset and condition, in bits.}
\label{tab:path_structure}
\small
\setlength{\tabcolsep}{3pt}
\begin{tabular}{lrrrrrrrr}
\toprule
& \multicolumn{4}{c}{AQUA} & \multicolumn{4}{c}{MATH} \\
\cmidrule(lr){2-5} \cmidrule(lr){6-9}
Condition & $N$ & $K_{\mathrm{used}}$ & Reuse (\%) & $\Hc_{\mathrm{path}}$ & $N$ & $K_{\mathrm{used}}$ & Reuse (\%) & $\Hc_{\mathrm{path}}$ \\
\midrule
\rowcolor{rpbred}
D-RPC & 9922 & 435 & 95.6 & 7.16 & 9562 & 868 & 90.9 & 7.85 \\
Empty & 9609 & 4495 & 53.2 & 10.90 & 9473 & 5210 & 45.0 & 11.66 \\
Random & 9875 & 1678 & 83.0 & 9.43 & 9599 & 2541 & 73.5 & 9.74 \\
Random-Global & 9876 & 1724 & 82.5 & 9.48 & 9581 & 2565 & 73.2 & 9.77 \\
Filtered-SFT & 8827 & 410 & 95.4 & 7.13 & 7482 & 700 & 90.6 & 7.74 \\
\bottomrule
\end{tabular}
\end{table}

\FloatBarrier

\subsection{Ablation Study}
We conduct three ablations to examine bank construction and guidance. All use identical SFT configurations. Results are averaged over 10 repeated executions with seed 42, following Section~\ref{subsec:impl-details}.

\noindent\textbf{Heterogeneous vs.\ homogeneous initialization.}
Table~\ref{tab:bank_size_cond_std} compares two seed sets: a \textbf{Heterogeneous} set sampled evenly across intent clusters to promote diversity and a \textbf{Homogeneous} set drawn from the largest clusters to limit diversity.
Heterogeneous initialization yields higher mean accuracy at 1k and 2k examples, while homogeneous initialization performs better from 3k to 5k.

\begin{table}[t]
\centering
\begin{minipage}[t]{0.35\textwidth}
\centering
\caption{Effect of bank size on GSM8K Acc and format validity.}
\label{tab:bank_size_paths_std}
\small
\setlength{\tabcolsep}{3pt}
{
\renewcommand{\arraystretch}{1.13}
\begin{tabular}{S[table-format=3.0] S[table-format=3.0] c c}
\toprule
{Size} & {Paths} & {Acc. (\%)} & {FV (\%)} \\
\midrule
50  & 49  & \res{83.19}{0.8}  & \res{83.32}{0.8} \\
75  & 75  & \bres{84.34}{0.8} & \bres{84.61}{0.8} \\
100 & 97  & \res{83.28}{0.6}  & \res{83.41}{0.7} \\
125 & 119 & \res{82.90}{0.7}  & \res{83.08}{0.7} \\
150 & 145 & \res{82.91}{0.5}  & \res{82.97}{0.5} \\
\bottomrule
\end{tabular}
}
\end{minipage}%
\hfill
\begin{minipage}[t]{0.60\textwidth}
\centering
\caption{Heterogeneous vs.\ homogeneous bank initialization on GSM8K across training set sizes.}
\label{tab:bank_size_cond_std}
\small
\setlength{\tabcolsep}{2.5pt}
\resizebox{\linewidth}{!}{%
\begin{tabular}{S[table-format=4.0]
                S[table-format=3.0] S[table-format=3.0]
                c c c c}
\toprule
{Bank} & \multicolumn{2}{c}{Paths}
       & \multicolumn{2}{c}{Acc (\%)}
       & \multicolumn{2}{c}{FV (\%)} \\
\cmidrule(lr){2-3} \cmidrule(lr){4-5} \cmidrule(lr){6-7}
{Size} & {hetero} & {homo} & {hetero} & {homo} & {hetero} & {homo} \\
\midrule
1000 & 50  & 57  & \bres{83.03}{0.7} & \res{82.30}{0.6}  & \bres{83.23}{0.7} & \res{82.37}{0.6} \\
2000 & 97  & 93  & \bres{83.24}{0.8} & \res{83.12}{0.3}  & \bres{83.38}{0.8} & \res{83.21}{0.4} \\
3000 & 151 & 133 & \res{84.18}{0.5}  & \bres{84.36}{0.4} & \res{84.27}{0.5}  & \bres{84.46}{0.4} \\
4000 & 188 & 186 & \res{83.83}{0.7}  & \bres{84.87}{0.6} & \res{83.95}{0.7}  & \bres{84.95}{0.6} \\
5000 & 243 & 252 & \res{84.58}{0.7}  & \bres{84.94}{0.6} & \res{84.69}{0.7}  & \bres{85.18}{0.7} \\
\bottomrule
\end{tabular}%
}
\end{minipage}
\end{table}
\vspace{0.2em}

\noindent\textbf{Bank guidance vs.\ no bank.}
Table~\ref{tab:all_metrics_with_delta} compares full D-RPC with Freeform generation under the shared teacher and training settings. D-RPC achieves higher accuracy across the reported benchmarks. This comparison evaluates the combined effect of path guidance and output schema.

\noindent\textbf{Bank size: the coverage--compression trade-off.}
Table~\ref{tab:bank_size_paths_std} varies the number of seed questions from 50 to 150 while fixing the SFT training set at 2k.
Mean accuracy peaks at 84.34\% with 75 paths from 75 seed questions, and is lower with 49 paths or 119 to 145 paths.
This pattern is consistent with the theoretical trade-off between bank size and coverage.

\subsection{Empirical Validation of the Bound}
\label{subsec:bound-validation}

Table~\ref{tab:bound-components} reports the key quantities from Theorem~\ref{thm:main} measured on AQUA with both student models.
All NLL values are per-token (length-normalized), consistent with the bounded-loss Assumption~\ref{assump:bounded-nll} and the training objective.
Test NLL measures agreement with teacher-generated targets. Task accuracy measures agreement with gold answers.
We approximate the Gibbs risk $\widehat L_S(\mathsf Q_S)$ by the point estimate $\widehat L_S(\hat\theta_S)$, standard for tight posteriors~\citep{dziugaite2017computing}.

\begin{table}[t]
\centering
\caption{Empirical bound components on AQUA (per-token NLL). 
  $K_{\mathrm{bank}}=42$ canonical paths for D-RPC. Test NLL uses teacher-generated targets.}
\label{tab:bound-components}
\small
\setlength{\tabcolsep}{3pt}
\begin{tabular}{l l
                S[table-format=1.2]
                S[table-format=1.2]
                S[table-format=1.2]
                S[table-format=1.2]
                S[table-format=4.0]
                S[table-format=2.1]}
\toprule
Model & Method
  & {$\widehat L_S(0)$}
  & {$\widehat L_S(\hat\theta_S)$}
  & {$L_{\mathrm{test}}(\hat\theta_S)$}
  & {Gap}
  & {$\|\hat\theta_S\|^2$}
  & {Acc.\ (\%)} \\
\midrule
\multirow{3}{*}{{Qwen 3 1.7B ($r=64$)}}
  & D-RPC & \bfseries 2.68 & 0.30           & 0.42           & \bfseries 0.11 & 2792           & \bfseries 74.0 \\
  & CoT   & 2.70           & 0.35           & 0.55           & 0.20           & 2760           & 61.4 \\
  & Free  & 3.29           & \bfseries 0.27 & \bfseries 0.38 & 0.11           & \bfseries 2732 & 63.0 \\
\midrule
\multirow{3}{*}{{Llama 3.1 8B ($r=128$)}}
  & D-RPC & \bfseries 2.01 & \bfseries 0.28 & 0.47           & 0.19           & 6318           & \bfseries 63.4 \\
  & CoT   & 2.02           & 0.38           & 0.67           & 0.29           & 6264           & 61.8 \\
  & Free  & 2.72           & 0.29           & \bfseries 0.45 & \bfseries 0.17 & \bfseries 6216 & 59.4 \\
\bottomrule
\end{tabular}
\end{table}

\noindent\textbf{Banking reduces frozen-model NLL.}
$\widehat L_S(0)$ is lowest for D-RPC on both models (Qwen: 2.68 vs.\ 2.70/3.29; Llama: 2.01 vs.\ 2.02/2.72),
showing better fit to banked targets for the two tested students.
This is consistent with the frozen-model alignment established in Proposition~\ref{prop:frozen-align} and shows that banking can reduce $N$ in Theorem~\ref{thm:main}.

\noindent\textbf{Generalization gap and accuracy.}
D-RPC achieves the highest test accuracy on both models (Qwen: 74.0\%; Llama: 63.4\%).
CoT consistently has the largest generalization gap (Qwen: 0.20; Llama: 0.29), while D-RPC and Freeform are comparable, consistent with Theorem~\ref{thm:main} predicting a tighter bound under moderate bank size.

\noindent\textbf{Adapter norms.}
The adapter norms $\|\hat\theta_S\|^2$ are comparable across strategies within each model, with Freeform slightly lower than D-RPC on both.
For fixed prior and regularization parameters, lower frozen-model NLL reduces $N$ in Eq.~\ref{eq:def-N}. The bank-size contribution appears separately in $M$ (Eq.~\ref{eq:def-M}).

\section{Limitations}
\label{sec:limitations}

\noindent\textbf{Evaluation scope.}
All experiments use a single teacher, GPT-5.1, and two student architectures, Llama~3.1 8B Instruct and Qwen~3 1.7B, across five reasoning benchmarks.
While D-RPC's gains are consistent across both student families and all task types tested, we have not verified whether the improvements transfer to other teacher models, larger or smaller student scales, or domains beyond math and commonsense reasoning such as code generation and scientific QA.

\noindent\textbf{Computational overhead.}
D-RPC's pipeline incurs about $2{\times}$ the teacher-query cost of standard CoT, across three components: a categorization pass over all training questions, a seed-set pass on 5\% of the training data to build the initial bank, and a rationale-generation pass conditioned on retrieved paths.
DCoT generates three CoT rationales per question, with a teacher-query cost of roughly $3{\times}$ that of standard CoT.
The categorization pass supports seed sampling and intent clustering, making it a possible target for reducing teacher-query overhead.
This cost is paid offline and does not affect student inference, though it may limit applicability where teacher API costs dominate.

\noindent\textbf{Statistical scope.}
The main comparisons use repeated executions with seed 42, so their uncertainty estimates are conditional on this seed.
Appendix~\ref{app:aqua_multiseed} separately reports AQUA results across ten training seeds.

\noindent\textbf{Path selection and supervision quality.}
Our retrieval procedure ranks candidate reasoning paths by embedding similarity after category and intent routing.
A natural extension is to combine this matching criterion with rationale quality and the student's learning needs.
Token-level uncertainty could guide supervision weighting, while quality estimates could inform path selection and bank refinement~\citep{kim2026explainownwords,yan2025efficientcotdistillationselfguided}.
Evaluating these extensions on MATH could clarify when greater structural reuse translates into higher student accuracy.

\noindent\textbf{Theoretical generality.}
The PAC-Bayes generalization analysis in Section~\ref{sec:theory} is derived under LoRA-based SFT with Gaussian weight priors.
Extending the bound to full fine-tuning, alternative parameter-efficient methods such as adapters or prefix tuning, or non-Gaussian prior families remains open and would broaden the theoretical applicability of the framework.

\section{Conclusion}
\label{sec:conclusion}
We proposed D-RPC, which compresses the teacher's reasoning space into a compact, reusable bank of canonical paths before distillation.
By routing semantically similar questions to shared paths and dynamically absorbing novel solutions, the bank provides consistent yet diverse supervision without verbose templates.
A PAC-Bayes generalization bound formalizes the trade-off between bank size and coverage through an explicit $\log K_{\mathrm{bank}}$ term. The bank-size ablation favors an intermediate size, while the frozen-model NLL diagnostic shows better fit to banked targets.
Across five reasoning benchmarks, D-RPC achieves the highest average accuracy among the evaluated methods for both student models.
More broadly, our results suggest that structuring the supervision distribution before training is an effective and lightweight strategy for improving SLM reasoning.

\section*{Acknowledgements and Disclosures}
\noindent\textbf{Funding.} This material is based in part upon work supported by Natural Sciences and Engineering Research Council of Canada (NSERC) Discovery Grant Program (RGPIN-2024-03954), NSERC Alliance - Alberta Innovates Advance Program ALLRP 602246-24, and Digital Research Alliance of Canada (alliancecan.ca).

\noindent\textbf{Competing interests.} The authors declare no competing interests.

\clearpage
\bibliographystyle{plainnat}
\bibliography{reference}

\clearpage
\appendix
\section{Additional Related Work}
\noindent\textbf{Reasoning prompting and structured inference.}
Prompting methods such as few-shot prompting~\citep{few_shot} and Chain-of-Thought (CoT)\allowbreak~\citep{chain_of_thought} improve LLM reasoning by eliciting intermediate steps. Subsequent work explored richer structured rationales during inference, including self-consistency~\citep{selfconsistencyimproveschainthought}, self-questioning~\citep{self_questioning}, ReAct~\citep{ReAct}, and Tree-of-Thoughts~\citep{Tree_Of_Thoughts}.
A recurring limitation is that these methods often induce miscellaneous, divergent, and sometimes unfaithful rationales~\citep{tamber-etal-2025-benchmarking,CoTAccuracy_Difficulty,wang2023scottselfconsistentchainofthoughtdistillation}.
These issues are particularly pronounced for SLMs. The limited capacity and context windows make them more sensitive to verbosely structured prompting schemes, where the model must explicitly construct and traverse the data structure during generation.

\noindent\textbf{Rationale used for supervised fine-tuning.}
Model-generated rationales are widely used as an intermediate supervision signal in supervised fine-tuning (SFT), as they can improve both task performance and interpretability~\citep{ExplainYourself, High_Quality_Rationales}.
To enhance rationale quality, MoRSD studies how to select rich, high-quality, informative rationales~\citep{yan2025efficientcotdistillationselfguided}. Thought Anchors explores identifying key rationales that dominate the decision of final results~\citep{bogdan2025thoughtanchorsllmreasoning}. Earlier work on Symbolic CoT Distillation showed that even sub-1B-parameter students can acquire step-by-step reasoning when trained on diverse rationales sampled from a much larger teacher, and that diversity in the sampled chains, not their individual likelihood, is the dominant factor in distillation success~\citep{li2024symbolicchainofthoughtdistillationsmall}.
However, teacher rationales can vary substantially in clarity, step granularity, and correctness across questions and difficulty levels~\citep{variation_across_tasks, LLMWeakOnStepsComplex}, leading to a decrease in distillation performance.
Existing studies suggest that such rationale divergence has been linked to reduced faithfulness and training instability~\citep{InstabilityReasoning, CoTNotFaithful, HumanPreferenceEvaluation}.

\section{Additional Prompting Details}
\subsection{Prompt Templates}
\label{app:prompt-templates}
We apply four prompting strategies to the teacher model GPT-5.1 to generate rationales.

\noindent\textbf{Chain-of-Thought Prompt.}
\begin{promptverb}
Think step by step in a few precise steps (no more than six sentences) to solve the problem.

Then output ONLY a compact JSON object of the form:
{
  "rationale": "<explanation>",
  "ans": <numeric_answer>
}

Rules:
- "ans" must be a number, not a string.
- No additional text before or after the JSON.
\end{promptverb}

\noindent\textbf{Stage 1 Initial Bank Construction Prompt(Ours).}
\begin{promptverb}
You are a rigorous but concise math tutor.
You solve math problems carefully and explain your reasoning briefly and clearly.
Avoid unnecessary prose; show only the key steps needed to reach the answer.

You are solving a math or word problem.
- Decompose the problem into abstract reasoning steps.
- For each step, use a general action name (no question-specific words).
- Each step must show actual computation.

Question: <QUESTION_TEXT>
Category: <ROUTE_CATEGORY>
Intent: <INTENT_TEXT>

Output ONLY valid JSON in the following format, and AVOID question-specific wording in reasoning_path keys:
{
  "route": {
    "difficulty": <1|2|3>,
    "budget": <Follow Budget Contract>,
    "reasoning_path": ["<DescriptiveStepName1>", "<DescriptiveStepName2>", ...]
  },
  "rationale": {
    "<DescriptiveStepName1>": "Step1: <...> Step2: <...> Step3: <...>",
    "<DescriptiveStepName2>": "Step1: <...> Step2: <...>",
    ...
  },
  "ans": <numeric>
}

Hard rules:
- Output must match the required JSON schema exactly.
- "route" must be a dictionary containing exactly: difficulty, budget, and reasoning_path.
- "reasoning_path" must be a list of strings representing ordered steps.
- The keys in "rationale" must match the strings in "reasoning_path" exactly.
- Output JSON only; no extra text.

Budget contract:
- Difficulty 1 implies Budget 2.
- Difficulty 2 or 3 implies Budget 3.
- len(reasoning_path) must be <= budget.

Reasoning key naming policy:
1) Keys must be TitleCase letters only (A--Z, a--z). No spaces, underscores, or digits.
2) Keys must be high-level descriptive summaries of the action taken; avoid question-specific wording.
3) Do NOT use generic sequential names (e.g., StepOne, CalculationOne).

Rationale requirements:
- For each reasoning_path entry, write a detailed multi-step explanation in one string.
- Use Step1:, Step2:, Step3:, ... labels; use as many substeps as needed (substeps do NOT count toward budget).
- Each substep should be concise and computational.
\end{promptverb}

\noindent\textbf{Stage 2 SFT Reasoning Extraction Prompt(Ours).}
\begin{promptverb}
You are a structured reasoning tutor.
You will be given:
- A math question
- One or more routing plans including:
  - Category
  - Intent
  - Budget
- ReasoningPathOptions (candidate paths)

Pick the best routing plan from the options and follow it if adequate; otherwise,
refine only the reasoning_path conservatively.

Question:
<QUESTION_TEXT>

Route:
{
  "category": "<CATEGORY>",
  "intent": ["<INTENT>"],
  "difficulty": <1|2|3>,
  "budget": <BUDGET>,
  "reasoning_path_options": [
    ["<PathOption1>"],
    ["<PathOption2>", "<PathOption3>"]
  ]
}

Task Instructions:

1) Assess reasoning_path suitability:
- If reasoning_path is missing and reasoning_path_options is provided, select the best option.
- If the selected reasoning_path can solve the problem, keep it unchanged.
- If not, create a revised reasoning_path conservatively:
  * Use Title-Case letters only (A--Z, a--z)
  * No digits, underscores, or question-specific wording
  * Length must be <= budget
- Keep category, intent, difficulty, and budget unchanged.

2) Generate detailed rationale:
- For each entry in the chosen reasoning_path, write a detailed explanation.
- Label sub-steps as Step1:, Step2:, Step3:, ...
- Use as many sub-steps as needed (sub-steps do NOT count toward budget).
- Each sub-step must include explicit computation or derivation when applicable.

3) Final answer:
- Compute the numeric answer and place it in "ans".

Output ONLY valid JSON (no markdown, no extra text):
{
  "route": {
    "category": "<CATEGORY>",
    "intent": ["<INTENT>"],
    "difficulty": <1|2|3>,
    "budget": <BUDGET>,
    "reasoning_path": ["<FinalPath1>", "<FinalPath2_optional>"]
  },
  "rationale": {
    "<DescriptiveStepName1>": "Step1: <...> Step2: <...> Step3: <...>",
    "<DescriptiveStepName2>": "Step1: <...> Step2: <...>",
    ...
  },
  "ans": <numeric>
}

Requirements:
- Each reasoning_path entry must appear exactly once as a key in rationale.
- Keys in rationale must match reasoning_path entries exactly.
- Do not exceed the budget in number of reasoning_path steps.
- Rationale values must be single strings with Step1:, Step2:, ... labels.
\end{promptverb}

\noindent\textbf{Freeform Prompt.}
\begin{promptverb}
You are a rigorous but concise math tutor.
You solve math problems carefully and explain your reasoning briefly and clearly.
Avoid unnecessary prose; show only the key steps needed to reach the answer.

You are solving a math or word problem.
- Decompose the problem into abstract reasoning steps.
- For each step, use a general action name (no question-specific words).
- Each step must show actual computation.

Question:
<QUESTION_TEXT>

Your task:
Solve the question and return ONLY a valid JSON object.

Output ONLY valid JSON with freely chosen reasoning paths:
{
  "rationale": {
    "<DescriptiveStepName1>": "<concise derivation>",
    "<DescriptiveStepName2>": "<concise derivation>",
    ...
  },
  "ans": <numeric>
}

Rules:
- Freely choose up to three reasoning paths with high-level, TitleCase names (letters only; no spaces, digits, or underscores).
- Keys must be human-readable action labels (e.g., PlanComputation, CombineTotals); avoid generic names like ReasoningPath1 and avoid question-specific wording.
- Keep each rationale concise and focused on the computation or logic.
- If a second path is unnecessary, omit it.
- "ans" must be numeric (no strings, units, or words).
- Do not include any markdown or extra text; return JSON only.
\end{promptverb}

\noindent\textbf{SuperCorrect XML Prompt.}
\begin{promptverb}
You are a rigorous but concise math tutor.
You solve math problems carefully and explain your reasoning briefly and clearly.
Avoid unnecessary prose; show only the key steps needed to reach the answer.

Transform the solution of the following math problem into a step-by-step XML format.
Each step should be enclosed within tags like <Step1> </Step1>.
For each step, determine if the step is challenging or tricky; if so, add a detailed explanation
enclosed within <Key> </Key> as annotations to help the student understand the step correctly.
After all steps, summarize the common solution pattern to help generalize to similar problems
within <Generalized> </Generalized>. Finally, present the final answer enclosed within
<Answer> </Answer>.

Problem:
<QUESTION_TEXT>
\end{promptverb}
\newpage
\section{Qualitative Examples of the D-RPC Pipeline}
\label{app:pipeline-examples}

This section provides simplified, illustrative examples that trace each stage of D-RPC on concrete problem instances.
The JSON structures and path names follow the actual format produced by our pipeline (see prompt templates in Appendix~\ref{app:prompt-templates}); specific values are abbreviated for clarity.
We use GSM8K arithmetic questions as the primary running example and include a StrategyQA instance to illustrate cross-domain applicability.

\noindent\textbf{Concrete example.}
For GSM8K, the bank might contain the entry:
\smallskip
\\
\noindent\colorbox{rpbred}{\parbox{\dimexpr\columnwidth-2\fboxsep}{%
\small
\textbf{Category:} Arithmetic \quad \textbf{Intent:} unit-rate computation\\
\textbf{Reasoning Paths:} \{[compute rate, combine quantities],\ [find unit rate, multiply by count]\}%
}}
\smallskip

\noindent Two canonical strategies for the same intent, enough for coverage yet few enough for consistency.

\subsection{Reasoning-Path Bank Structure}
\label{app:ex-bank}

Table~\ref{tab:bank-snapshot} shows a three-slot excerpt from the reasoning-path bank $\mathcal{B}$ (defined in Section~\ref{subsec:stage1}) to illustrate its structure.
Each row corresponds to a (category, canonical intent) pair and lists the reasoning paths $\Pi \in \mathcal{B}(c,\tilde{t})$ associated with that slot.
Paths are ordered lists of abstract, reusable step names following the TitleCase naming policy enforced by the Stage~1 prompt in Appendix~\ref{app:prompt-templates}.
The full bank in our GSM8K experiments contains 50--150 paths depending on configuration (see the ablation in Table~\ref{tab:bank_size_paths_std}).

\begin{table}[t]
\caption{Three-slot excerpt from the reasoning-path bank $\mathcal{B}$.
Each slot maps a (category, canonical intent) pair to one or more canonical paths.
This is an illustrative fragment; the full bank contains 50--150 paths.}
\label{tab:bank-snapshot}
\centering
\small
\setlength{\tabcolsep}{3pt}
\begin{tabular}{lll}
\toprule
Category & Canonical Intent & Paths in $\mathcal{B}(c,\tilde t)$ \\
\midrule
Arithmetic & unit-rate &
  \parbox[t]{4.0cm}{\raggedright
    $\Pi_1$: \texttt{[ComputeUnitRate,}\\
    \phantom{$\Pi_1$: }\texttt{MultiplyByCount]}\\[2pt]
    $\Pi_2$: \texttt{[IdentifyRatio,}\\
    \phantom{$\Pi_2$: }\texttt{ScaleToTarget]}} \\[6pt]
Arithmetic & total-cost &
  \parbox[t]{4.0cm}{\raggedright
    $\Pi_3$: \texttt{[ComputeItemCost,}\\
    \phantom{$\Pi_3$: }\texttt{SumComponents]}} \\[6pt]
Commonsense & \parbox[t]{1.8cm}{temporal\\feasibility} &
  \parbox[t]{4.0cm}{\raggedright
    $\Pi_4$: \texttt{[RecallKeyFacts,}\\
    \phantom{$\Pi_4$: }\texttt{CheckTimeline,}\\
    \phantom{$\Pi_4$: }\texttt{DeduceAnswer]}} \\
\bottomrule
\end{tabular}
\end{table}

\subsection{Question Categorization}
\label{app:ex-categorization}

Before any reasoning takes place, the teacher's categorization function $f_{\mathrm{cat}}$ from Section~\ref{subsec:stage1}, Step~1, assigns a category--intent pair $(C_i, T_i)$ to every training question.
The examples below show this output for three questions spanning two domains.

\begin{promptverb}
Q1: "A baker sells cupcakes for $3 each. If
  she sold 48 on Monday and 32 on Tuesday,
  how much did she earn in total?"
  -> Category: Arithmetic
     Intent:   total-cost aggregation

Q2: "Tom drives 60 miles in 1.5 hours. How
  far can he drive in 4 hours at the same
  speed?"
  -> Category: Arithmetic
     Intent:   unit-rate computation

Q3: "Did Albert Einstein ever visit the
  Moon?"
  -> Category: Commonsense
     Intent:   temporal feasibility check
\end{promptverb}

\noindent
These labels serve only as keys for bank lookup and path retrieval; they are not included in the student's training data.
Q1 and Q2 share a category but differ in intent, so they are routed to different bank slots.
Q3 illustrates that the same categorization mechanism applies to non-mathematical domains.

\subsection{Retrieved Paths vs.\ Freeform-Generated Paths}
\label{app:ex-retrieved-vs-freeform}

A central claim of this work is that unconstrained teacher generation introduces unnecessary supervision variance: structurally similar questions receive structurally different rationales.
To illustrate this contrast, we show the reasoning paths produced by D-RPC and by the Freeform baseline for two unit-rate questions that share the same underlying solution strategy.

\medskip\noindent
Q2: ``Tom drives 60 miles in 1.5 hours. How far can he drive in 4 hours?''

\smallskip
\begin{tabular}{@{}l@{\ }l@{}}
\textbf{D-RPC:} & \texttt{[ComputeUnitRate, MultiplyByCount]}\\
\textbf{Freeform:} & \texttt{[FindSpeed, MultiplyDistance]}\\
\end{tabular}

\medskip\noindent
Q4: ``A factory makes 150 widgets in 5 hours. How many in 8 hours?''

\smallskip
\begin{tabular}{@{}l@{\ }l@{}}
\textbf{D-RPC:} & \texttt{[ComputeUnitRate, MultiplyByCount]}\\
\textbf{Freeform:} & \texttt{[SetUpProportion, CrossMultiply]}\\
\end{tabular}

\medskip\noindent
Under D-RPC, both questions match the bank slot \texttt{Arithmetic/unit-rate} and retrieve the same canonical path $\Pi_1$, so the student observes a single consistent template.
Under Freeform, the teacher improvises a fresh structure for each question, choosing \texttt{FindSpeed} for Q2 and \texttt{SetUpProportion} for Q4, producing valid but structurally incompatible rationales.
When the student trains on both examples via $\mathcal{L}_{\mathrm{SFT}} = -\log p_\theta(Y \mid X)$, the conflicting output structures act as label noise: the model must allocate capacity to reconcile multiple solution formats for the same problem type, rather than reinforcing a single reusable pattern.
This is precisely the supervision-variance reduction that the entropy-floor analysis in Proposition~\ref{prop:entropy-baseline} formalizes.

\subsection{Bank-Guided Rationale Generation}
\label{app:ex-teacher-output}

Once a path is retrieved, the teacher receives the question and the candidate paths via the Stage~2 prompt $f_{\mathrm{teach}}$ described in Section~\ref{subsec:stage2}, with the full prompt in Appendix~\ref{app:prompt-templates}, and generates the supervision tuple $Y = (\Pi, R, A)$.
Below is the complete input--output pair for Q2.

\begin{promptverb}
Input to teacher (Stage 2):
  Question: "Tom drives 60 miles in 1.5 hours.
    How far can he drive in 4 hours at the
    same speed?"
  Route:
  {
    "category": "Arithmetic",
    "intent": ["unit-rate computation"],
    "difficulty": 1,
    "budget": 2,
    "reasoning_path_options": [
      ["ComputeUnitRate", "MultiplyByCount"],
      ["IdentifyRatio", "ScaleToTarget"]
    ]
  }

Teacher output (Y):
  {
    "route": {
      "category": "Arithmetic",
      "intent": ["unit-rate computation"],
      "difficulty": 1,
      "budget": 2,
      "reasoning_path":
        ["ComputeUnitRate", "MultiplyByCount"]
    },
    "rationale": {
      "ComputeUnitRate":
        "Step1: Distance = 60 miles,
          Time = 1.5 hours.
         Step2: Speed = 60 / 1.5 = 40 mph.",
      "MultiplyByCount":
        "Step1: Target time = 4 hours.
         Step2: Distance = 40 * 4 = 160 miles."
    },
    "ans": 160
  }
\end{promptverb}

\noindent
The teacher selects path option $\Pi_1 = \texttt{[ComputeUnitRate, MultiplyByCount]}$ from the two candidates provided, then instantiates each abstract step with concrete sub-step computations to form the rationale $R$.
The student's training example is $(X, Y)$ where $Y = (\Pi, R, A)$: the reasoning path, the detailed rationale, and the final answer.

\noindent\textbf{Contrast: Freeform generation on the same question.}
Under the Freeform prompt, $f_{\mathrm{teach}}$ receives no candidate paths and must generate both the solution structure and content from scratch:

\begin{promptverb}
Freeform teacher output (Q2):
  {
    "rationale": {
      "FindSpeed":
        "Step1: Speed = 60 / 1.5 = 40 mph.",
      "MultiplyDistance":
        "Step1: Distance = 40 * 4 = 160 miles."
    },
    "ans": 160
  }
\end{promptverb}

\noindent
The final answer is identical, but the output lacks the \texttt{route} field entirely and uses the ad-hoc path name \texttt{FindSpeed} rather than the canonical \texttt{ComputeUnitRate}.
When a different unit-rate question is posed, the Freeform teacher may invent yet another structure, for instance \texttt{SetUpProportion}~$\to$~\texttt{CrossMultiply} for Q4 above, whereas the bank-guided teacher consistently selects from the same small set of canonical paths.
Over the full training set, this structural consistency reduces the variance of the supervision signal that the student must fit, which is the mechanism underlying D-RPC's empirical gains.

\section{Supplementary Setting}
\subsection{Detailed Datasets}
\label{app:detailed-datasets}

\begin{table}[h]
\centering
\caption{Datasets used for reasoning transfer training and evaluation. Task type, answer format, and the number of training and evaluation samples for each dataset are reported.}
\label{tab:datasets}
\small
\setlength{\tabcolsep}{3pt}
\renewcommand{\arraystretch}{1.05}
\begin{tabular}{@{}lllrr@{}}
\toprule
\textbf{Dataset} & \textbf{Task} & \textbf{Answer} & \textbf{\#Train} & \textbf{\#Eval} \\
\midrule
GSM8K      & Math       & Numeric & 7{,}473  & 1{,}319 \\
AQUA       & Math       & Letter  & 10{,}000 & 254 \\
StrategyQA & Commonsense & T/F    & 2{,}061  & 229 \\
AI2ARC     & Science    & Letter  & 6{,}230  & 1{,}557 \\
MATH       & Comp.\ math & Numeric & 10{,}000 & 5{,}000 \\
\bottomrule
\end{tabular}
\end{table}

\textbf{GSM8K}~\citep{gsm8k} is a grade-school mathematics dataset consisting of arithmetic word problems that require multi-step numerical reasoning.


\textbf{AQUA}~\citep{AQUA} is a multiple-choice mathematics reasoning dataset that requires selecting the correct option based on symbolic and logical reasoning.

\textbf{StrategyQA}~\citep{strategyQA} is a commonsense reasoning dataset in which each question requires implicit multi-hop reasoning over factual knowledge, with Boolean (True/False) answers.

\subsection{Training Details}
\label{training_details}
The main supervised fine-tuning (SFT) comparisons are conducted using an identical training configuration to ensure fair and controlled comparisons across different reasoning supervision strategies. We fine-tune \texttt{Llama-3.1-8B-Instruct} as the student model using the AdamW optimizer with a learning rate of $1\times10^{-4}$ and a weight decay of $0.01$. The maximum sequence length is set to 2048 tokens.

Training is performed for two epochs with a per-device batch size of 2 and gradient accumulation over 8 steps, resulting in an effective batch size of 16. We employ a linear learning rate scheduler with a warmup ratio of $3\%$, and apply gradient clipping with a maximum norm of 1.0 to stabilize training.

To improve parameter efficiency, we adopt LoRA-based fine-tuning. LoRA adapters are applied to the query, key, value, and output projection layers with default rank $r=64$, scaling factor $\alpha=128$, and a dropout rate of $0.05$. All experiments are conducted using bfloat16 precision. The Llama 3.1 8B checkpoints in Table~\ref{tab:bound-components} use $r=128$.

Model checkpoints are saved every 250 training steps, and training logs are recorded every 25 steps. Each repeated execution uses seed 42, except for the training seeds specified in Appendix~\ref{app:aqua_multiseed}. The evaluation protocol in Section~\ref{subsec:impl-details} describes the source of variation. Unless otherwise specified, no limit is imposed on the number of training samples. We use the same hyperparameter configuration for Qwen 3 1.7B.

\subsection{Supplementary Experiments}
\begin{figure*}[t]
  \centering
  \includegraphics[width=\linewidth]{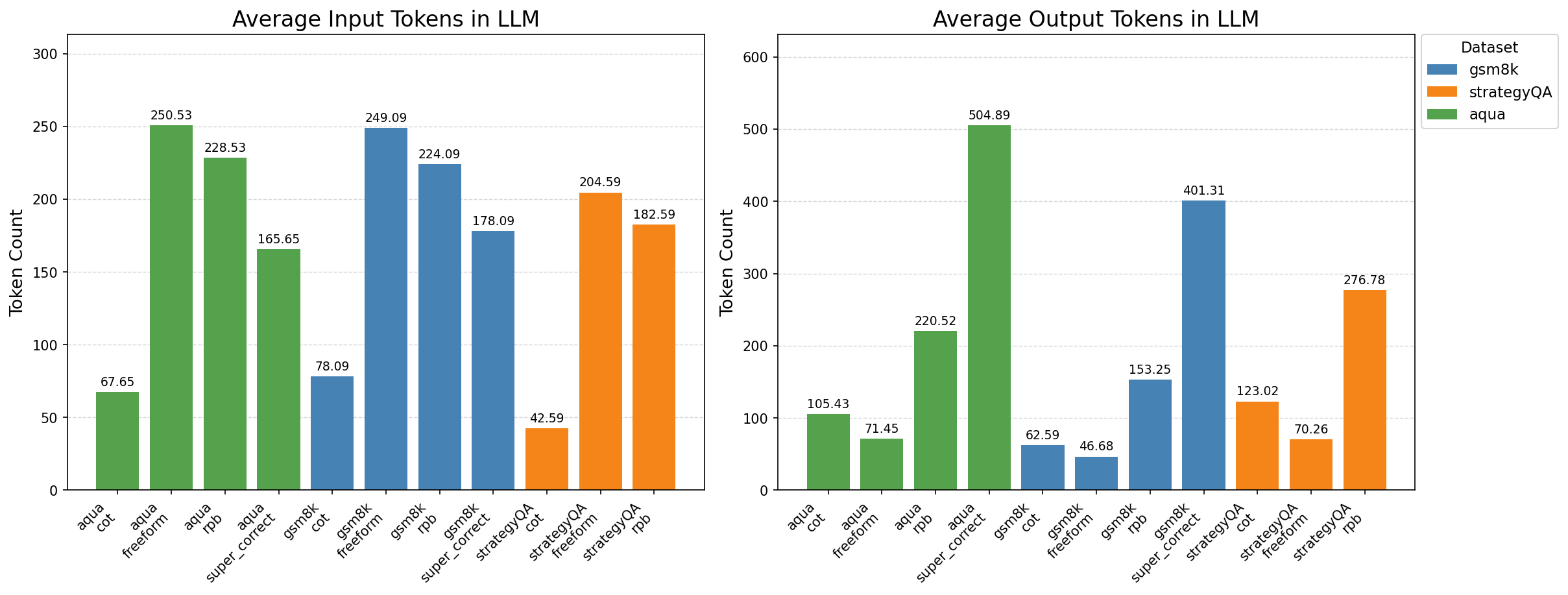}
  \vspace{0.8em}
  \includegraphics[width=\linewidth]{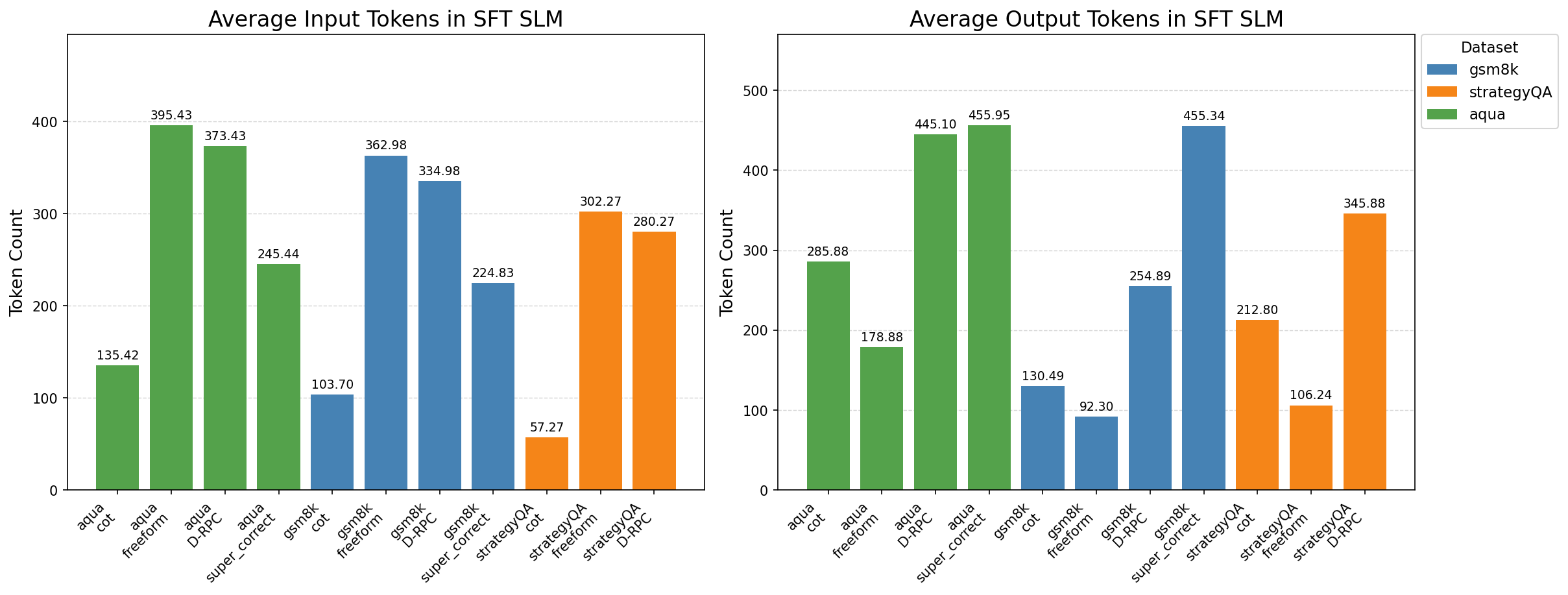}
  \caption{Token usage analysis under different reasoning supervision strategies. 
  \textbf{Top:} Average input and output token usage of the teacher LLM during rationale generation.
  \textbf{Bottom:} Average input and output token usage of the SFT student model at evaluation time.}
  \label{fig:token_analysis}
\end{figure*}

\noindent\textbf{Token cost: D-RPC improves efficiency relative to template-heavy supervision.}
\label{app:token_cost}
Figure~\ref{fig:token_analysis} compares average input/output token usage for both the teacher LLM during rationale generation and the SFT student at evaluation time. While D-RPC includes retrieved reasoning paths and thus uses substantially more teacher input tokens than CoT and a comparable amount to Freeform, it substantially reduces teacher output length relative to SuperCorrect, yielding an overall reduction of roughly 33--35\% in total teacher tokens per example across the math datasets. On the student side, D-RPC-trained models produce more explicit rationales than CoT and Freeform on every dataset, while staying shorter on average than SuperCorrect supervision (with the gap largest on GSM8K), avoiding the output-length explosion that increases inference latency.

\begin{table}[t]
\caption{AQUA performance across training fractions under different LoRA ranks ($r{=}128$ vs.\ $r{=}64$). Bold denotes the higher accuracy between the two ranks for each method and training fraction, FV denotes Format Validity.}
\centering
\small
\setlength{\tabcolsep}{2.8pt}
\renewcommand{\arraystretch}{0.95}
\begin{tabular}{lcccc lcccc}
\toprule
\multicolumn{5}{c}{Train (\%) = 10} & \multicolumn{5}{c}{Train (\%) = 60} \\
\cmidrule(lr){1-5}\cmidrule(lr){6-10}
Strategy & r128 Acc & r128 FV & r64 Acc & r64 FV &
Strategy & r128 Acc & r128 FV & r64 Acc & r64 FV \\
\midrule
cot            & \textbf{64.96} & 67.35 & 62.20          & 64.75 &
cot            & 60.24          & 61.45 & \textbf{62.20} & 63.45 \\
freeform       & 61.02          & 61.75 & \textbf{62.20} & 67.23 &
freeform       & \textbf{59.45} & 60.89 & 57.48          & 59.35 \\
D-RPC(ours)            & 61.42          & 63.93 & \textbf{63.78} & 66.80 &
D-RPC(ours)            & \textbf{64.17} & 64.94 & 63.78          & 65.85 \\
super\_correct & 59.06          & 61.70 & \textbf{59.84} & 62.98 &
super\_correct & \textbf{58.66} & 62.45 & 57.87          & 61.44 \\
\midrule
\multicolumn{5}{c}{Train (\%) = 20} & \multicolumn{5}{c}{Train (\%) = 70} \\
\cmidrule(lr){1-5}\cmidrule(lr){6-10}
cot            & 60.24          & 60.96 & \textbf{64.96} & 67.90 &
cot            & \textbf{66.93} & 69.11 & 61.42          & 61.51 \\
freeform       & 57.87          & 59.27 & \textbf{58.27} & 60.41 &
freeform       & 59.06          & 59.52 & \textbf{63.78} & 65.06 \\
D-RPC(ours)            & \textbf{63.78} & 68.22 & 62.20          & 63.31 &
D-RPC(ours)            & 65.75          & 66.80 & \textbf{68.50} & 70.45 \\
super\_correct & \textbf{57.87} & 61.11 & 56.30          & 58.90 &
super\_correct & \textbf{58.66} & 63.36 & 57.87          & 61.25 \\
\midrule
\multicolumn{5}{c}{Train (\%) = 30} & \multicolumn{5}{c}{Train (\%) = 80} \\
\cmidrule(lr){1-5}\cmidrule(lr){6-10}
cot            & \textbf{64.96} & 66.80 & 62.20          & 63.45 &
cot            & \textbf{65.35} & 65.61 & 61.81          & 62.06 \\
freeform       & \textbf{59.06} & 60.73 & 57.09          & 59.67 &
freeform       & \textbf{62.99} & 63.75 & 59.45          & 61.13 \\
D-RPC(ours)            & \textbf{66.93} & 68.00 & 62.60          & 65.16 &
D-RPC(ours)            & \textbf{66.14} & 68.57 & 62.99          & 64.00 \\
super\_correct & \textbf{61.02} & 64.44 & 54.72          & 57.69 &
super\_correct & \textbf{64.96} & 66.39 & 58.27          & 61.18 \\
\midrule
\multicolumn{5}{c}{Train (\%) = 40} & \multicolumn{5}{c}{Train (\%) = 90} \\
\cmidrule(lr){1-5}\cmidrule(lr){6-10}
cot            & \textbf{66.54} & 68.15 & 63.78          & 64.29 &
cot            & \textbf{65.35} & 65.61 & 64.17          & 64.68 \\
freeform       & \textbf{60.63} & 62.35 & 58.27          & 58.96 &
freeform       & \textbf{58.27} & 60.16 & 57.87          & 58.57 \\
D-RPC(ours)            & 66.14          & 67.74 & \textbf{67.72} & 69.80 &
D-RPC(ours)            & 67.72          & 68.80 & \textbf{68.11} & 70.33 \\
super\_correct & \textbf{61.02} & 63.45 & 59.45          & 61.70 &
super\_correct & \textbf{62.99} & 65.38 & 60.24          & 63.09 \\
\midrule
\multicolumn{5}{c}{Train (\%) = 50} & \multicolumn{5}{c}{Train (\%) = 100} \\
\cmidrule(lr){1-5}\cmidrule(lr){6-10}
cot            & 62.60          & 63.10 & \textbf{66.14} & 66.67 &
cot            & \textbf{67.72} & 67.98 & 65.75          & 66.01 \\
freeform       & 57.48          & 58.87 & \textbf{57.87} & 58.80 &
freeform       & \textbf{64.57} & 65.08 & 56.30          & 57.20 \\
D-RPC(ours)            & \textbf{64.17} & 65.46 & 61.42          & 64.46 &
D-RPC(ours)            & \textbf{69.69} & 70.68 & 66.14          & 67.20 \\
super\_correct & \textbf{57.87} & 61.44 & 54.33          & 56.54 &
super\_correct & \textbf{60.63} & 62.86 & 56.30          & 59.17 \\
\bottomrule
\end{tabular}

\label{tab:aqua_r128_r64_fv}
\end{table}

\noindent\textbf{LoRA Setting: How LoRA Rank and Alpha impact the results.}
\label{app:lora_impact}
Table~\ref{tab:aqua_r128_r64_fv} compares the rankings of LoRA at the same training ratios of r=64 and r=128. The impact of ranking depends on the training data volume. At low data ratios (10-20\%), r=64 outperforms r=128, consistent with the phenomenon that the low-ranked adapter acts as an implicit regularizer to resist fitting noise. As the training ratio increases, the trend shifts to capacity-constrained adaptation, where r=128 achieves higher accuracy at 80-100\% data ratios. This suggests that r=64 underfits once sufficient supervision is provided. At all data ratios, D-RPC remains the strongest policy, achieving the best accuracy across most fractions and peaking at r=128 (69.69\% accuracy) while maintaining good format validity. This indicates that its structured guidance is robust to ranking selection and can benefit from additional adapter capacity as the amount of data increases.

\noindent\textbf{Implications for reasoning space compression.}
The bank-size ablation favors an intermediate size, while the effect of initialization diversity varies with the training set size. These results support managing the balance between path reuse and coverage when constructing the bank.

\subsection{AQUA results across training seeds}
\label{app:aqua_multiseed}
Table~\ref{tab:aqua_multiseed} reports accuracy on all 254 AQUA test questions using training seeds 43 to 52 and the LoRA configuration in Section~\ref{subsec:impl-details}.
Empty omits candidate paths, while Random and Random-Global use independent draws sampled uniformly from the reasoning path bank.
Filtered-SFT retains only D-RPC supervision tuples with correct final answers.
The relative ordering of Empty and Random differs between the two students.

\begin{table}[htbp]
\centering
\caption{AQUA test accuracy (\%) across ten training seeds. Entries report means and sample standard deviations.}
\label{tab:aqua_multiseed}
\small
\begin{tabular}{lcc}
\toprule
Condition & Llama 3.1 8B Instruct & Qwen 3 1.7B \\
\midrule
Empty & \res{66.93}{1.41} & \res{72.83}{2.83} \\
Random & \res{66.34}{2.40} & \res{74.21}{1.95} \\
Random-Global & \res{66.30}{2.15} & \res{73.43}{1.37} \\
Filtered-SFT & \res{68.35}{2.18} & \res{75.16}{2.10} \\
\bottomrule
\end{tabular}
\end{table}
\FloatBarrier

\subsection{Teacher answer errors and filtering}
\label{app:teacher_errors}
Table~\ref{tab:teacher_errors} reports final-answer errors in the D-RPC supervision collected for each dataset, and Table~\ref{tab:teacher_control_accuracy} compares teacher accuracy across path controls.
Filtered-SFT removes the 1173 incorrect-answer tuples from AQUA and 2518 from MATH before student training, with student results in Table~\ref{tab:path_controls}.

\begin{table}[htbp]
\centering
\begin{minipage}[t]{0.49\textwidth}
\centering
\caption{Teacher final-answer errors in collected D-RPC supervision.}
\label{tab:teacher_errors}
\small
\setlength{\tabcolsep}{4pt}
\begin{tabular}{lrrr}
\toprule
Dataset & Tuples & Errors & Error (\%) \\
\midrule
GSM8K & 7473 & 370 & 5.0 \\
AQUA & 10000 & 1173 & 11.7 \\
StrategyQA & 2061 & 397 & 19.3 \\
AI2ARC & 6230 & 97 & 1.6 \\
MATH & 10000 & 2518 & 25.2 \\
\bottomrule
\end{tabular}
\end{minipage}
\hfill
\begin{minipage}[t]{0.47\textwidth}
\centering
\caption{Teacher final-answer accuracy (\%) under the path controls.}
\label{tab:teacher_control_accuracy}
\small
\begin{tabular}{lrr}
\toprule
Condition & AQUA & MATH \\
\midrule
\rowcolor{rpbred}
D-RPC & 88.3 & 74.8 \\
Empty & 85.7 & 72.3 \\
Random & 88.2 & 73.7 \\
Random-Global & 88.1 & 73.6 \\
\bottomrule
\end{tabular}
\end{minipage}
\end{table}
\FloatBarrier

\subsection{Additional Random-Global executions}
\label{app:randglobal_extra}
Table~\ref{tab:randglobal_extra} reports an additional ten executions of Random-Global on the Fir cluster using seed 42 and all 5000 MATH test questions.
These executions use the training configuration and greedy decoding protocol of Table~\ref{tab:path_controls}, with a 2048-token output limit and generation batch size one.

\begin{table}[htbp]
\centering
\caption{Additional Random-Global accuracy on MATH (\%), reported as mean and sample standard deviation.}
\label{tab:randglobal_extra}
\small
\begin{tabular}{cc}
\toprule
Llama 3.1 8B Instruct & Qwen 3 1.7B \\
\midrule
\res{48.06}{0.38} & \res{60.22}{0.39} \\
\bottomrule
\end{tabular}
\end{table}
\FloatBarrier

\newpage
\section{Proof of Proposition~\ref{prop:entropy-baseline}}
\label{app:entropy-proof}

\begin{proof}
By the chain rule for conditional entropy,
\begin{align*}
&\Hc(Y\mid X)=\Hc(\Pi,R,A\mid X)=\Hc(\Pi\mid X)+\Hc(R\mid X,\Pi)+\Hc(A\mid X,\Pi,R).
\end{align*}
Under Assumption~\ref{assump:bank-support}, for each \(x\), the conditional distribution \(\Pi\mid X=x\) is supported on at most
\(K_{\mathrm{bank}}\) values, hence \(\Hc(\Pi\mid X=x)\le \log K_{\mathrm{bank}}\).
Taking expectation over \(X\) yields \(\E[\Hc(\Pi\mid X)]\le \log K_{\mathrm{bank}}\).
Substituting into the chain rule gives Eq.~\eqref{eq:entropy-baseline}.
\end{proof}

\section{Second-order PAC-Bayes: proof sketches for the bounded-loss route}
\label{app:pacbayes-proof}

For completeness, we state the intermediate results referenced in Section~\ref{sec:theory} and used in the derivation of Theorem~\ref{thm:main}.

\begin{proposition}[Bernstein-type PAC-Bayes for bounded losses]
\label{prop:bernstein-pacbayes}
Under Assumption~\ref{assump:bounded-nll}, for any prior \(\mathsf P\) independent of \(S\) and any \(\delta\in(0,1)\),
with probability at least \(1-\delta\) over \(S\sim\mathcal D^n\), simultaneously for all posteriors \(\mathsf Q\),
\begin{equation}
L(\mathsf Q)
\le
\widehat L_S(\mathsf Q)
+
\sqrt{
\frac{
2\,\widehat V_S(\mathsf Q)\big(\KL(\mathsf Q\|\mathsf P)+\ln\frac{1}{\delta}\big)
}{n}
}
+
\frac{
c\,\tau\big(\KL(\mathsf Q\|\mathsf P)+\ln\frac{1}{\delta}\big)
}{n}.
\label{eq:bernstein-pacbayes}
\end{equation}
where \(c>0\) is a universal numerical constant and \(\widehat V_S(\mathsf Q):=\E_{\theta\sim\mathsf Q}[\frac1n\sum_i(\ell(\theta;Z_i)-\widehat L_S(\theta))^2]\).
\end{proposition}

\begin{lemma}[Mean--variance domination for bounded losses]
\label{lem:mean-var}
Under Assumption~\ref{assump:bounded-nll},
\(\widehat V_S(\mathsf Q)\le \tau\,\widehat L_S(\mathsf Q)\).
\label{eq:mean-var-dom}
\end{lemma}

\begin{corollary}[Mean-dependent second-order PAC-Bayes]
\label{cor:mean-dep-pacbayes}
Substituting Lemma~\ref{lem:mean-var} into Proposition~\ref{prop:bernstein-pacbayes}: with probability at least \(1-\delta\),
\begin{equation}
L(\mathsf Q)
\le
\widehat L_S(\mathsf Q)
+
\sqrt{
\frac{
2\,\tau\,\widehat L_S(\mathsf Q)
\big(\KL(\mathsf Q\|\mathsf P)+\ln\frac{1}{\delta}\big)
}{n}
}
+
\frac{
c\,\tau\big(\KL(\mathsf Q\|\mathsf P)+\ln\frac{1}{\delta}\big)
}{n}.
\label{eq:mean-dep-pacbayes}
\end{equation}
\end{corollary}

\begin{corollary}[Explicit \(K_{\mathrm{bank}}\)-dependence]
\label{cor:K-explicit}
Under Assumptions~\ref{assump:bounded-nll}--\ref{assump:near-floor}, the same bound holds with \(\widehat L_S(\mathsf Q)\) replaced by \(M\) (Eq.~\ref{eq:def-M}) in the square-root term.
\label{eq:K-explicit}
\end{corollary}

\subsection{Proof sketch of Proposition~\ref{prop:bernstein-pacbayes}}
Proposition~\ref{prop:bernstein-pacbayes} is a standard Bernstein-type PAC-Bayes inequality for bounded losses.
A typical proof follows these steps:

\begin{enumerate}[leftmargin=*, itemsep=3pt]
\item \textbf{Bernstein mgf bound for bounded variables.}
For a bounded random variable \(U\in[0,\tau]\), one upper bounds
\(\log \E[\exp(\lambda(\E[U]-U))]\) in terms of \(\lambda\), \(\tau\), and \(\Var(U)\),
yielding a Bernstein-style concentration inequality.
\item \textbf{Change of measure (Donsker--Varadhan).}
For any measurable \(f\),
\[
\E_{\theta\sim \mathsf Q}[f(\theta)]
\le
\KL(\mathsf Q\|\mathsf P) + \log \E_{\theta\sim \mathsf P}[\exp(f(\theta))].
\]
\item \textbf{High-probability conversion.}
Apply Markov's inequality and a standard argument to obtain a statement that holds with probability at least \(1-\delta\),
introducing \(\ln(1/\delta)\), and obtain a bound simultaneously for all \(\mathsf Q\).
\item \textbf{Variance-sensitive form.}
Collect terms to obtain a deviation term scaling with \(\sqrt{\widehat V_S(\mathsf Q)}\) plus a lower-order
\(O((\KL+\ln(1/\delta))/n)\) term, where boundedness introduces the factor \(\tau\) in the latter.
\end{enumerate}

\subsection{Proof of Lemma~\ref{lem:mean-var}}
\begin{proof}
Fix \(\theta\) and define \(U\) as a random variable taking values \(\ell(\theta;Z_i)\) under the empirical distribution over indices
\(i\sim\mathrm{Unif}\{1,\dots,n\}\).
Under Assumption~\ref{assump:bounded-nll}, \(U\in[0,\tau]\), hence \(U^2\le \tau U\). Therefore,
\[
\Var(U)=\E[U^2]-\E[U]^2
\le \tau \E[U]-\E[U]^2
\le \tau \E[U].
\]
Noting that \(\E[U]=\widehat L_S(\theta)\), we obtain
\[
\frac1n\sum_{i=1}^n\big(\ell(\theta;Z_i)-\widehat L_S(\theta)\big)^2
= \Var(U)
\le \tau\,\widehat L_S(\theta).
\]
Finally, take expectation over \(\theta\sim\mathsf Q\) to get \(\widehat V_S(\mathsf Q)\le \tau\,\widehat L_S(\mathsf Q)\).
\end{proof}

\subsection{Derivation of Corollary~\ref{cor:mean-dep-pacbayes}}
Corollary~\ref{cor:mean-dep-pacbayes} follows immediately by substituting Lemma~\ref{lem:mean-var}
(Eq.~\eqref{eq:mean-var-dom}) into Proposition~\ref{prop:bernstein-pacbayes}
(Eq.~\eqref{eq:bernstein-pacbayes}).

\subsection{Derivation of Corollary~\ref{cor:K-explicit}}
Starting from Corollary~\ref{cor:mean-dep-pacbayes} (Eq.~\eqref{eq:mean-dep-pacbayes}), apply
Assumption~\ref{assump:near-floor} (Eq.~\eqref{eq:near-floor}) and Proposition~\ref{prop:entropy-baseline}
(Eq.~\eqref{eq:entropy-baseline}) to upper bound \(\widehat L_S(\mathsf Q)\) by
\(\log K_{\mathrm{bank}} + \E[\Hc(R\mid X,\Pi)] + \E[\Hc(A\mid X,\Pi,R)] + \varepsilon\).
Substituting this into Eq.~\eqref{eq:mean-dep-pacbayes} yields Eq.~\eqref{eq:K-explicit}.

\section{Norm control and Gaussian KL: proofs}
\label{app:kl-proof}

\subsection{Proof of Proposition~\ref{prop:norm-control}}
\begin{proof}
By optimality of \(\hat\theta_S\),
\[
F_S(\hat\theta_S)\le F_S(0)
\quad\Longrightarrow\quad
\widehat L_S(\hat\theta_S)+\frac{\lambda}{2}\|\hat\theta_S\|_2^2 \le \widehat L_S(0).
\]
Since \(\widehat L_S(\hat\theta_S)\ge 0\), we conclude \(\frac{\lambda}{2}\|\hat\theta_S\|_2^2 \le \widehat L_S(0)\),
which implies Eq.~\eqref{eq:norm-control}.
\end{proof}

\subsection{Proof of Proposition~\ref{prop:kl-bound}}
\begin{proof}
For isotropic Gaussians \(\mathcal N(\mu,\sigma_{\mathrm{post}}^2 I_m)\) and \(\mathcal N(0,\sigma_0^2 I_m)\),
the KL divergence has the closed form
\[
\KL
=
\frac{\|\mu\|_2^2}{2\sigma_0^2}
+
\frac{m}{2}\Big(\rho-1-\ln\rho\Big),
\qquad
\rho=\frac{\sigma_{\mathrm{post}}^2}{\sigma_0^2}.
\]
Setting \(\mu=\hat\theta_S\) yields the equality in Eq.~\eqref{eq:kl-bound}.
The inequality follows by substituting Proposition~\ref{prop:norm-control} into the mean-shift term
\(\|\hat\theta_S\|_2^2/(2\sigma_0^2)\).
\end{proof}

\section{Formal justification of frozen-model alignment under banking}
\label{app:frozen-align}

This appendix provides the formal proofs behind Proposition~\ref{prop:frozen-align} (stated in Section~\ref{sec:theory}).
The goal is to justify the following condition by a cross-entropy decomposition and designable sufficient conditions.

\begin{assumption}[Frozen-model alignment under banking (restated)]
\label{assump:frozen-align}
\(
\E_{S\sim\mathcal D_{\mathrm{bank}}^n}[\widehat L_S(0)]
\le
\E_{S\sim\mathcal D_{\mathrm{free}}^n}[\widehat L_S(0)].
\)
\end{assumption}

\subsection{Population form and cross-entropy decomposition}

Define the frozen-model population cross-entropy in regime \(\star\in\{\mathrm{free},\mathrm{bank}\}\) by
\begin{equation*}
\mathcal L_0^{(\star)}
:=
\E_{X}\E_{Y\sim p^{(\star)}(\cdot\mid X)}\big[-\log p_0(Y\mid X)\big],
\, Y=(\Pi,R,A).
\label{eq:pop-cross-entropy}
\end{equation*}
Assuming integrability so that \(\E_{S\sim(\mathcal D_\star)^n}[\widehat L_S(0)]=\mathcal L_0^{(\star)}\),
Assumption~\ref{assump:frozen-align} follows from
\begin{equation}
\mathcal L_0^{(\mathrm{bank})} \le \mathcal L_0^{(\mathrm{free})}.
\label{eq:pop-frozen-align}
\end{equation}

For any target conditional \(p^\star(\cdot\mid x)\),
\begin{equation*}
\E_{Y\!\sim \! p^\star\!(\cdot\mid x)}[\!-\!\log p_0(Y\!\!\mid\! x)]
\!=\!
\Hc\!\big(p^\star(\!\cdot\!\mid\!x)\big)\!+\KL\!\big(p^\star(\!\cdot\!\mid x)\|p_0(\!\cdot\!\mid \!x)\big),
\label{eq:cross-entropy-decomp}
\end{equation*}
and averaging over \(X\) yields
\begin{equation}
\mathcal L_0
=
\E[\Hc(Y\mid X)]
+
\E\!\left[\KL\!\big(p^\star(\cdot\mid X)\|p_0(\cdot\mid X)\big)\right].
\label{eq:cross-entropy-decomp-avg}
\end{equation}
Thus, Eq.~\eqref{eq:pop-frozen-align} is equivalent to
\begin{equation}
\Delta \Hc + \Delta \KL \le 0,
\label{eq:deltaH-deltaKL}
\end{equation}
where \(\Delta \Hc\) and \(\Delta \KL\) denote the bank-minus-free differences of the two terms in
Eq.~\eqref{eq:cross-entropy-decomp-avg}.
Proposition~\ref{prop:entropy-baseline} implies the \(\Pi\)-banking mechanism reduces (or upper-bounds) the path-uncertainty component
through \(\Hc(\Pi\mid X)\le \log K_{\mathrm{bank}}\), while Assumption~\ref{assump:frozen-align} additionally requires that any increase
in mismatch to \(p_0\) does not offset the entropy reduction.

\subsection{Sufficient designable conditions}

\noindent\textbf{Condition D1 (Bank as projection under the frozen model).}
Let \(\mathfrak P_{\mathrm{bank}}(K_{\mathrm{bank}})\) denote the family of conditional distributions supported on a
\(K_{\mathrm{bank}}\)-template path bank:
\[
\!\mathfrak P_{\mathrm{bank}}(\!K_{\mathrm{bank}})
\!\!:=\!\!
\Big\{
p(\!\cdot\!\mid x)\!:\! \mathrm{supp}(\Pi\!\mid\!\! X\!=\!x)\!\subseteq\! \mathcal B,|\mathcal B|\!=\!\!K_{\mathrm{bank}\!}\!
\Big\},
\]
with no restriction on \(R,A\) given \((X,\Pi)\).
Define \(p^\star_{\mathrm{bank}}\) as
\begin{equation}
p^\star_{\mathrm{bank}}
\in
\argmin_{p\in\mathfrak P_{\mathrm{bank}}(K_{\mathrm{bank}})}
\ \E_{X}\E_{Y\sim p(\cdot\mid X)}[-\log p_0(Y\mid X)].
\label{eq:bank-projection}
\end{equation}
Then \(\mathcal L_0^{(\mathrm{bank})}\le \mathcal L_0^{(\mathrm{free})}\) holds by construction.

\noindent\textbf{Condition D2 (Canonicalization improves frozen likelihood pointwise).}
If there exists a map \(c\) on targets such that for all \((x,y)\),
\begin{equation}
-\log p_0(c(y)\mid x) \le -\log p_0(y\mid x),
\label{eq:canonicalization}
\end{equation}
and the banked regime replaces \(Y\) by \(c(Y)\) (or restricts support to such canonical forms), then
Eq.~\eqref{eq:pop-frozen-align} follows immediately by taking expectations.

\noindent\textbf{Condition D3 (Entropy drop dominates mismatch increase).}
From Eq.~\eqref{eq:deltaH-deltaKL}, it suffices that
\[
\E[\Hc_{\mathrm{free}}(Y\mid X)]-\E[\Hc_{\mathrm{bank}}(Y\mid X)]
\ \ge\
\E[\KL_{\mathrm{bank}}-\KL_{\mathrm{free}}],
\]
i.e., any increase in mismatch to \(p_0\) is smaller than the entropy reduction due to banking.

Under any of Conditions D1--D3, Assumption~\ref{assump:frozen-align} is justified, establishing Proposition~\ref{prop:frozen-align}.

\section{Proof sketch of Theorem~\ref{thm:main}}
\label{app:main-proof}

\begin{proof}[Proof sketch]
Fix \(\delta\in(0,1)\).
On the event where Corollary~\ref{cor:K-explicit} holds (probability at least \(1-\delta\)),
set \(\mathsf Q=\mathsf Q_S\) and apply Proposition~\ref{prop:kl-bound} to upper bound
\(\KL(\mathsf Q_S\|\mathsf P)\) by
\[
\frac{1}{\sigma_0^2\lambda}\widehat L_S(0) + \frac{m}{2}\big(\rho-1-\ln\rho\big).
\]
Substitute this into the deviation terms in Eq.~\eqref{eq:K-explicit}.
Then use the definitions of \(M\) (Eq.~\eqref{eq:def-M}) and \(N\) (Eq.~\eqref{eq:def-N}) to obtain Eq.~\eqref{eq:main}.
\end{proof}

\clearpage
\section*{NeurIPS Paper Checklist}

\begin{enumerate}

\item {\bf Claims}
    \item[] Question: Do the main claims made in the abstract and introduction accurately reflect the paper's contributions and scope?
    \item[] Answer: \answerYes{}
    \item[] Justification: The abstract and introduction present D-RPC, the bank-size/coverage trade-off, and evaluation on five benchmarks with two students. Matched retrieval is presented as an additional refinement.

\item {\bf Limitations}
    \item[] Question: Does the paper discuss the limitations of the work performed by the authors?
    \item[] Answer: \answerYes{}
    \item[] Justification: Section~\ref{sec:limitations} discusses the single-teacher evaluation, teacher-query overhead, fixed-seed uncertainty, path selection, and the scope of the LoRA-based analysis.

\item {\bf Theory Assumptions and Proofs}
    \item[] Question: For each theoretical result, does the paper provide the full set of assumptions and a complete (and correct) proof?
    \item[] Answer: \answerYes{}
    \item[] Justification: Section~\ref{sec:theory} states the assumptions, with proofs and derivations in Appendices~\ref{app:entropy-proof}, \ref{app:pacbayes-proof}, \ref{app:kl-proof}, \ref{app:frozen-align}, and \ref{app:main-proof}.

\item {\bf Experimental Result Reproducibility}
    \item[] Question: Does the paper fully disclose all the information needed to reproduce the main experimental results of the paper to the extent that it affects the main claims and/or conclusions of the paper (regardless of whether the code and data are provided or not)?
    \item[] Answer: \answerYes{}
    \item[] Justification: Algorithm~\ref{alg:drpc_compact}, Section~\ref{sec:experiments}, and Appendices~\ref{app:prompt-templates} and \ref{training_details} specify the method, prompts, data, training, and evaluation settings.

\item {\bf Open access to data and code}
    \item[] Question: Does the paper provide open access to the data and code, with sufficient instructions to faithfully reproduce the main experimental results, as described in supplemental material?
    \item[] Answer: \answerYes{}
    \item[] Justification: The repository provides code, data, reproduction instructions, and a comparison of released predictions with Table~\ref{tab:all_metrics_with_delta}.

\item {\bf Experimental Setting/Details}
    \item[] Question: Does the paper specify all the training and test details (e.g., data splits, hyperparameters, how they were chosen, type of optimizer, etc.) necessary to understand the results?
    \item[] Answer: \answerYes{}
    \item[] Justification: Section~\ref{sec:experiments} and Appendix~\ref{training_details} give dataset splits, training configurations, checkpoint selection, and decoding, with additional AQUA seeds in Appendix~\ref{app:aqua_multiseed}.

\item {\bf Experiment Statistical Significance}
    \item[] Question: Does the paper report error bars suitably and correctly defined or other appropriate information about the statistical significance of the experiments?
    \item[] Answer: \answerYes{}
    \item[] Justification: Main-table standard deviations describe ten executions with seed 42 and nondeterministic GPU execution. Table~\ref{tab:aqua_multiseed} separately reports means and sample standard deviations across training seeds 43 to 52.

\item {\bf Experiments Compute Resources}
    \item[] Question: For each experiment, does the paper provide sufficient information on the computer resources (type of compute workers, memory, time of execution) needed to reproduce the experiments?
    \item[] Answer: \answerNo{}
    \item[] Justification: GPU execution times and the total compute budget are not fully documented. Section~\ref{sec:limitations} and Appendix~\ref{app:token_cost} report teacher-query overhead and token costs.

\item {\bf Code Of Ethics}
    \item[] Question: Does the research conducted in the paper conform, in every respect, with the NeurIPS Code of Ethics \url{https://neurips.cc/public/EthicsGuidelines}?
    \item[] Answer: \answerYes{}
    \item[] Justification: The work uses public benchmarks and a commercial teacher API under its terms, without recruiting human participants. Funding and competing interests are disclosed after the conclusion.

\item {\bf Broader Impacts}
    \item[] Question: Does the paper discuss both potential positive societal impacts and negative societal impacts of the work performed?
    \item[] Answer: \answerYes{}
    \item[] Justification: D-RPC can make reasoning capabilities more accessible through smaller student models. These models can also inherit teacher biases and produce incorrect rationales, which may mislead users who rely on their explanations.

\item {\bf Safeguards}
    \item[] Question: Does the paper describe safeguards that have been put in place for responsible release of data or models that have a high risk for misuse (e.g., pretrained language models, image generators, or scraped datasets)?
    \item[] Answer: \answerNA{}
    \item[] Justification: The work studies distillation for established math and commonsense benchmarks and does not introduce a new general-purpose pretrained model or scraped dataset.

\item {\bf Licenses for existing assets}
    \item[] Question: Are the creators or original owners of assets (e.g., code, data, models), used in the paper, properly credited and are the license and terms of use explicitly mentioned and properly respected?
    \item[] Answer: \answerNo{}
    \item[] Justification: The original dataset and model sources are credited in Sections~\ref{sec:method} and \ref{sec:experiments}. Individual asset licenses and API terms are not explicitly listed.

\item {\bf New Assets}
    \item[] Question: Are new assets introduced in the paper well documented and is the documentation provided alongside the assets?
    \item[] Answer: \answerYes{}
    \item[] Justification: The released assets include installation and usage instructions, data schemas, prediction manifests, and supplementary experiment documentation.

\item {\bf Crowdsourcing and Research with Human Subjects}
    \item[] Question: For crowdsourcing experiments and research with human subjects, does the paper include the full text of instructions given to participants and screenshots, if applicable, as well as details about compensation (if any)?
    \item[] Answer: \answerNA{}
    \item[] Justification: All supervision is teacher-generated on public benchmarks, with no crowdsourcing or human participants.

\item {\bf Institutional Review Board (IRB) Approvals or Equivalent for Research with Human Subjects}
    \item[] Question: Does the paper describe potential risks incurred by study participants, whether such risks were disclosed to the subjects, and whether Institutional Review Board (IRB) approvals (or an equivalent approval/review based on the requirements of your country or institution) were obtained?
    \item[] Answer: \answerNA{}
    \item[] Justification: The work does not involve research with human participants.

\item {\bf Declaration of LLM usage}
    \item[] Question: Does the paper describe the usage of LLMs if it is an important, original, or non-standard component of the core methods in this work? If so, does the paper include any information regarding which LLMs were used and for what purpose?
    \item[] Answer: \answerYes{}
    \item[] Justification: GPT-5.1 performs categorization, reasoning-path elicitation, and rationale generation, while Llama 3.1 8B Instruct and Qwen 3 1.7B are the students. Their roles and prompts are described in Sections~\ref{sec:method} and \ref{sec:experiments} and Appendix~\ref{app:prompt-templates}.

\end{enumerate}

\end{document}